\documentclass[11pt]{article}
\usepackage[utf8]{inputenc}
\usepackage{amsmath, amssymb, amsthm}
\usepackage{geometry}
\usepackage{hyperref}
\usepackage{pgfplots}
\usepackage{authblk}
\usepackage{xcolor}
\usepackage{mleftright}\mleftright
\definecolor{niceRed}{RGB}{190,38,38}
\definecolor{Red2}{RGB}{219, 50, 54}
\definecolor{mgreen}{RGB}{160, 200, 140}
\definecolor{blueGrotto}{RGB}{5,157,192}
\definecolor{limeGreen}{HTML}{81B622}
\definecolor{myellow}{rgb}{0.88,0.61,0.14}
\definecolor{darkGreen}{HTML}{2E8B57}
\definecolor{navyBlueP}{HTML}{03468F}
\definecolor{Sepia}{HTML}{7F462C}
\definecolor{red2}{HTML}{1F462C}
\definecolor{orange2}{HTML}{FF8000}
\definecolor{mgray}{HTML}{ABB3B8}
\definecolor{lgray}{HTML}{E5E8E9}
\definecolor{myPurple}{RGB}{175,0,124}
\definecolor{mypurple2}{rgb}{0.8,0.62,1}
\definecolor{royalBlue}{HTML}{057DCD}
\definecolor{mpink}{HTML}{FC6C85}
\definecolor{lblue}{RGB}{74,144,226}
\definecolor{peagreen}{RGB}{152,193,39}
\definecolor{typ_navy}{HTML}{001f3f}
\definecolor{typ_blue}{HTML}{0074d9}
\definecolor{typ_aqua}{HTML}{7fdbff}
\definecolor{typ_teal}{HTML}{39cccc}
\definecolor{typ_eastern}{HTML}{239dad}
\definecolor{typ_purple}{HTML}{b10dc9}
\definecolor{typ_fuchsia}{HTML}{f012be}
\definecolor{typ_maroon}{HTML}{85144b}
\definecolor{typ_red}{HTML}{ff4136}
\definecolor{typ_orange}{HTML}{ff851b}
\definecolor{typ_yellow}{HTML}{ffdc00}
\definecolor{typ_olive}{HTML}{3d9970}
\definecolor{typ_green}{HTML}{2ecc40}
\definecolor{typ_lime}{HTML}{01ff70}
\definecolor{newgreen}{HTML}{83c702}
\definecolor{newpurp}{RGB}{97,96,121}
\definecolor[named]{Purple}{cmyk}{0.55,1,0,0.15}
\definecolor[named]{DarkBlue}{cmyk}{1,0.58,0,0.21}
\hypersetup{colorlinks,
    colorlinks = true,
    linkcolor=DarkBlue,
    citecolor=Sepia,
    urlcolor=Purple,
    filecolor=DarkBlue
    linktocpage = true,
}
\usepackage{enumitem}

\pgfplotsset{compat=1.18}

\usepackage{tikz}
\usepackage[capitalize,noabbrev]{cleveref}
\usepackage[style=alphabetic,natbib=true,maxbibnames=10]{biblatex}

\usetikzlibrary{arrows.meta,fit,positioning}
\pgfdeclarelayer{background}
\pgfsetlayers{background,main}

\usepackage{mathtools}
\newtheorem{theorem}{Theorem}
\newtheorem*{theorem*}{Theorem}
\newtheorem{lemma}{Lemma}
\newtheorem{proposition}{Proposition}

\newtheorem{definition}{Definition}

\newcommand{\R}{\mathbb{R}}
\newcommand{\E}{\mathbb{E}}
\newcommand{\norm}[1]{\left\| #1 \right\|}

\newcommand{\ip}[2]{\langle #1, #2 \rangle}

\newcommand{\cX}{\mathcal{X}}
\newcommand{\cD}{\mathcal{D}}

\newcommand{\cZ}{\mathcal{Z}}
\newcommand{\cO}{\mathcal{O}}

\newcommand{\cY}{\mathcal{Y}}

\newcommand{\eps}{\varepsilon}
\newcommand{\RRM}{\mathrm{RRM}}

\newcommand{\thetaPS}{\theta_{\mathrm{PS}}}
\newcommand{\thetaPO}{\theta_{\mathrm{PO}}}
\newcommand{\muPS}{\mu_{\mathrm{PS}}}

\newcommand{\poly}{\mathrm{poly}}

\DeclareMathOperator*{\argmin}{arg\,min}

\newcommand{\EVIGap}{\operatorname{Gap}_{\rm EVI}}

\newcommand{\cV}{\mathcal{V}}

\newcommand{\DPR}[2]{
\E_{z\sim\cD(#1)}\ell(z;#2)
}
\newcommand{\tDPR}{\mathrm{DPR}}

\usepackage[linesnumbered,ruled,lined,noend]{algorithm2e}
\input{algo_config}

\title{Stabilizing Performative Feedback Loops \\ with Minimal Model Deployments}

\date{}

\author[1]{Gabriele Farina}
\author[2]{Juan Carlos Perdomo}
\affil[1]{MIT EECS}
\affil[2]{New York University}

\begin{document}

\maketitle

\thispagestyle{empty}
\begin{abstract}
When algorithmic predictions inform people's decisions, the models we deploy are performative and  actively shape the data we see. This feedback loop between algorithms and their broader environments introduces a challenge in the mechanics of social prediction: If different predictive models induce different distributions, is it possible to efficiently learn a prediction rule that is optimal for the distribution that it induces? Formally, this solution concept is known as \emph{performative stability}. A core challenge in learning a performatively stable predictor is that, unlike supervised learning where distributions are fixed, the learner must deploy different predictors and observe their induced distributions.


The main contribution of our work is a new algorithmic procedure that, in the high-accuracy regime, finds a performatively stable model in nearly the minimum number of model deployments without making any assumptions regarding how predictions shape distributions. In particular, our procedure succeeds at finding a randomized performatively stable predictor using exponentially fewer model deployments than prior approaches. Our second main contribution is a structural result showing how this recent randomized notion of stability achieved by our algorithm can be derandomized into a single predictor satisfying the prior deterministic notion if one is willing to assume that the loss is well-conditioned and that performative effects are weak, as in early work in this area. On a technical level, our results come from building on an underexplored technical connection between performative stability and expected variational inequalities.
\end{abstract}

\clearpage
\thispagestyle{empty}
\tableofcontents

\clearpage
\setcounter{page}{1}
\section{Introduction}

The impacts of algorithmic forecasts on society are manifest and ubiquitous. 
Medical risk predictors influence physician decisions and patient behavior in ways that actively shape future health outcomes at a national scale. 
Forecasts of individual preferences on social media platforms shape people's true preferences and opinions.
And financial predictions determine market expectations influencing the true supply and demand patterns for underlying goods. 

In short, when we use predictions to inform decisions, these predictions become \textit{performative}. The models we deploy actively shape the data we see. While this idea has deep roots within economics and sociology dating back to the work of \citet{morgenstern1928wirtschaftsprognose} and \citet{merton1948self}, it was only recently formalized within computer science by Perdomo, Zrnic, Mendler-D\"{u}nner, and Hardt \cite{perdomo2020performative}, henceforth referred to as PZMH.

Performativity introduces a key challenge to the mechanics of social prediction. How could one possibly predict future outcomes if outcomes are moving targets that evolve with predictions themselves?
Theoretical investigations into this question started with the work of Herbert \citet{simon1954bandwagon,grunberg1954predictability}. 
More recently, there is renewed
interest in the problem due to the widespread presence of these feedback loops in  domains like algorithmic credit scoring, medicine, as well as language modeling, see e.g. \citet{hardt2026retraining}.

PZMH introduced performative stability as a solution to these algorithmic feedback loops. A model is performatively stable if it is risk-minimizing with respect to the joint distribution over features and outcomes that it induces. Stable models satisfy the property that if we deploy them and then fit a new model to their induced distribution, we get the same model back. Said otherwise, performatively stable models are fixed points of repeated retraining. Consequently, they prevent a common runaway feedback loop where developers respond to the distribution shift caused by model deployment by repeatedly refitting to the new distribution and redeploying.

Given their appeal, there has been significant interest in understanding when and why stable models can be learned. Initial papers in the area showed that if the underlying loss function is smooth, strongly convex, and predictions have only a very weak and smoothly varying influence on the distribution, classical stochastic optimization algorithms converge to stability \cite{mendler2020stochastic,drusvyatskiy2023stochastic}. Unfortunately, these restrictions on model influence are necessary in theory, yet patently false in practice \cite{dews,mendler2024engine}. If no restrictions are made regarding how predictions influence outcomes, finding a performatively stable model can be PPAD-hard \cite{anagnostides2026computational}, or provably impossible (since stable models may fail to exist).

\citet{farina2026stability} recently showed a way around this conundrum. If we allow the learner to output a \textit{randomized} predictor that mixes between different models rather than forcing them to deterministically make predictions with a single model, then stability is always efficiently achievable. In more detail, they show that the uniform mixture over the iterates produced by any no-regret online algorithm is performatively stable regardless of how models shape distributions. While this insight substantially expands the scope of problems for which one can guarantee stability, it comes with nontrivial deployment costs. To achieve an $\eps$-stable mixture via this no-regret approach, one must deploy, and observe how people respond to, a polynomial in $1/\eps$ many models.

Relative to supervised learning where the complexity of learning is governed by the number of samples needed to generalize and the computational cost of searching for a hypothesis, in performative contexts we also need to consider the number of deployments. In domains where data is abundant and computation is cheap, the organizational costs of placing a new model into production are often the main bottleneck. For instance, a large social media platform might collect hundreds of thousands of data points per day, yet they might only be able to deploy a new model a few dozen times per year. 
Motivated by this reality, we initiate the study of the following problem:
\begin{center}
    \emph{How many model deployments are necessary and sufficient \\ to stabilize feedback loops in performative prediction?}
\end{center}
This question of model deployments is also intimately tied to the question of how much \textit{randomness} is necessary to stabilize performative feedback loops. Motivated by the hardness of finding a deterministic stable model, \citet{farina2026stability} and \citet{perdomo2025revisiting} showed polynomially large mixtures suffice. However, it is unknown whether that much randomness is truly necessary.

Our work has two main results. First, we show that one can always find a performatively stable mixture that randomizes across $d\log(1/\eps)$ many models in $d\log(1/\eps)$ many deployments without making any restrictions regarding how predictions influence data. This is an exponential improvement over prior work. In the high-accuracy regime where $\eps\leq d^{-1}$, we prove a lower bound showing its deployment complexity is unimprovable up to a log factor in $d$. Our techniques behind this result come from establishing a new
connection between performative stability and recent advances in computational perspectives on the minimax theorem \cite{sigecom}. 

Our second main result is a structural insight relating the original deterministic definition of performative stability from PZMH and the randomized relaxation from \citet{farina2026stability}. Performatively stable mixtures are always achievable, yet require the learner to randomize across multiple models. We show that if the restrictive conditions studied in the early work are true (that is, that the objective is well-conditioned and similar predictions induce similar distributions), then the mixture can be derandomized into a single predictor that is itself stable. 

This result shows that focusing on this recent randomized notion of stability is, in a sense, without loss of generality, since it implies the more stringent deterministic notion in the regime where the deterministic notion was known to be achievable. 
Beyond establishing how these two definitions are more closely related than previously thought, we illustrate how this structural result enables a new algorithmic pathway to find deterministic performatively stable models.

\subsection{Our Contributions}

We describe our results using the performative prediction framework introduced in PZMH.

Let $\Theta \subset \R^d$ be a convex, compact set of parameters $\theta$ for predictive models $f_\theta$. We will refer to models $f_\theta$ and parameters $\theta$ interchangeably. Relative to supervised learning where examples $z$ are always drawn i.i.d. from a fixed distribution $\cD$, in performative prediction if the learner deploys a model $f_\theta$, they see samples $z$ drawn from the model's induced distribution $\cD(\theta)$. Here, $\cD(\cdot)$ is the distribution map. It is an unknown function from models $\Theta$ to the space of distributions over the data $\Delta(\cZ)$. The learner can simply deploy $\theta$ and collect samples $z \sim \cD(\theta)$.

If the data space $\cZ$ is equal to the product set of features and outcomes $\cZ = \cX \times \cY$, the distribution map $\cD(\cdot)$ formally captures how the choice of predictor $\theta$ influences both the marginal distribution over features as well as the conditional distribution of outcomes. Performativity therefore formally subsumes the setting known as strategic classification \cite{hardt2016strategic}.

Given a loss function $\ell(z;\theta)$, a model is performatively stable if it approximately minimizes the expected risk over the data that it induces. More formally, 
\begin{definition}[Performative Stability]
A model $\thetaPS \in \Theta$ is $\eps$-performatively stable if 
\begin{align*}
    \DPR{\thetaPS}{\thetaPS} - \min_{\theta \in \Theta}\DPR{\thetaPS}{\theta} \leq \eps 
\end{align*}
On the other hand, a distribution $\muPS$ over models in $\Theta$ is an $\eps$-performatively stable mixture if 
\begin{align*}
    \E_{\theta \sim \muPS}\left[\DPR{\theta}{\theta}\right] - \min_{\theta' \in \Theta} \E_{\theta \sim \muPS} \left[\DPR{\theta}{\theta'} \right] \leq \eps 
\end{align*}
\end{definition}
Recalling our earlier discussion, stable models are the fixed points of retraining. By definition,
\begin{align*}
    \thetaPS \in \argmin_{\theta \in \Theta} \E_{z \sim \cD(\thetaPS)} \ell(z;\theta).
\end{align*}
Performatively stable mixtures generalize the same desiderata from deterministic predictors to settings where institutions randomize across different models. Assume that the learner makes predictions by first drawing a model $\theta \sim \muPS$  from a mixture $\muPS$ and then sees data $z \sim \cD(\theta)$ drawn from the distribution induced by that \textit{specific} model $\theta$. Stable mixtures satisfy the property that when we inspect their performance over the pushforward distribution induced via this sampling process, no model could have achieved lower loss:
\begin{align*}
    \E_{\theta \sim \muPS}\left[\DPR{\theta}{\theta}\right] \leq \min_{\theta' \in \Theta} \E_{\theta \sim \muPS} \left[\DPR{\theta}{\theta'} \right]. 
\end{align*}

By establishing a reduction to online learning,\citet{farina2026stability} illustrated how one can always efficiently compute an $\eps$-performatively stable mixture with only $\poly(1/\eps)$ samples and model deployments. 
On a technical level, this result presented a significant departure from prior work in performative prediction in that it placed no restrictions on the distribution map $\cD(\cdot)$. In particular, almost all prior research in the area imposed a Lipschitzness condition on $\cD(\cdot)$, requiring that predictions influence outcomes only in a weak and smoothly varying fashion. These restrictions apply only to a very limited class of problems. Concretely, $\cD(\cdot)$ is discontinuous whenever people choose actions (e.g. decide to administer a drug or intervention) by thresholding the outputs of the predictor as is often the case in education or medicine. See \Cref{sec:technical_overview}.

In this work, we introduce a new algorithmic toolkit for finding performatively stable points that similarly avoids any continuity assumptions on $\cD(\cdot)$ while requiring exponentially fewer model deployments. In particular, we leverage a  technical connection between performative prediction and expected variational inequalities to prove the following algorithmic result. 

\begin{theorem*}[Informal]
Assume that the loss function $\ell(z;\theta)$ is convex where $\theta\in \R^d$ and let $\cD(\cdot)$ be any, possibly discontinuous distribution map. 

Then, a polynomial-time cutting-plane algorithm that repeatedly bisects the solution space finds an $\eps$-performatively stable mixture $\muPS$ supported on $\widetilde O(d\log(1/\eps))$ models while deploying only $\widetilde O(d\log(1/\eps))$ predictors and collecting $\widetilde O(d/\eps^2)$ samples. If $\ell(z;\theta)$ is strongly convex in $\theta$, the procedure succeeds using only $\widetilde O(d/\eps)$ samples. 
\end{theorem*}

As discussed previously, the costs of deploying a new model often far outweigh the costs of data collection. Our work uses recent advances in algorithmic game theory to overcome these deployment barriers. In particular, we reduce the question of finding a performatively stable mixture to that of solving an expected variational inequality problem. 

This reframing allows us to leverage (and extend) modern versions of the Ellipsoid-Against-Hope algorithm from Papadimitriou and Roughgarden to find stable mixtures with exponentially fewer deployments than were previously necessary \cite{papadimitriou2008computing,daskalakis2025efficient}. Complementing this upper bound, we prove a matching lower bound showing that for worst-case distribution maps $\cD(\cdot)$, $d\log(1/\eps)$ deployments are necessary.
We present a technical overview of these ideas in \Cref{sec:technical_overview}.

Moving on, our second main result establishes a new structural relationship between randomized and deterministic notions of performative stability. The theorem is stated for performative EVI-stability---a new notion of stability we introduce in our work which is a slight strengthening of performatively stable mixtures as defined in \citet{farina2026stability}. Our cutting-plane procedure returns mixtures which satisfy this slightly stronger notion (see \Cref{sec:technical_overview}).


\begin{theorem*}[Informal]
Assume that the distribution map is $L$-Lipschitz and that the loss function $\ell(z;\theta)$ is $\gamma$-strongly convex in $\theta$ and $\beta$-smooth in $z$, with $\rho \coloneqq L\beta /\gamma \leq 1$. 

Then, the mean $\theta_{\muPS} = \E_{\theta\sim \muPS}[\theta] \in \R^d$ of any $\eps$-performatively EVI-stable mixture $\muPS$ is itself an $\eps$-performatively stable model.
\end{theorem*}

The assumptions in the first sentence of the informal theorem above are exactly the ones considered in the initial papers on performative prediction by \cite{perdomo2020performative,mendler2020stochastic,drusvyatskiy2023stochastic} where the authors proved that various stochastic optimization algorithms converged to stability if $\rho < 1$. Recently, \citet{anagnostides2026computational} showed that, for $\rho$ bounded away from $1$, it becomes PPAD-hard to find a single performatively stable model. For $\rho=1$, it was not known either way whether stability was computationally tractable in finite samples.

In light of our first main theorem, this second result shows that one can always find an $\eps$-stable \textit{mixture} $\muPS$ even if $\rho$ is infinite while only deploying a few models.  However, if $\rho \leq 1$ then we can collapse the mixture supported on $d\log(1/\eps)$ models into a single model that is itself stable. Because mixtures are always finitely supported, one can easily derandomize them in linear time by computing the mean.

Interestingly, unlike prior work \cite{drusvyatskiy2023stochastic,mendler2020stochastic}, both the sample and deployment complexity of our procedures are independent of $(1-\rho)^{-1}$. They therefore speed up convergence to stability in the regime where $\rho$ is very close to, but just under 1. Furthermore, our second theorem holds even right at the threshold where $\rho=1$. This shows that one can always find a deterministic, performatively stable model right up until the critical threshold where they are computable by first computing a stable mixture and then collapsing it into a single predictor. This is the first positive finite-sample result for computing stable models in this $\rho=1$ regime.

\subsection{Problem Background and  Overview of Our Techniques}
\label{sec:technical_overview}

\paragraph{Technical Background.} In this section, we provide a concise summary of our techniques with the aim of conveying the intuition behind our results. We will present formal versions in later sections. To  contextualize our contributions, we start by very briefly overviewing the state of prior work in performative prediction. We present a more comprehensive discussion of the related literature in \Cref{sec:discussion}.

Motivated by the ubiquity of predictive feedback loops in recommender systems, allocation problems, and more recently language modeling (amongst others), \citet{perdomo2020performative} studied the following repeated risk minimization, or RRM, procedure in which models are fit on the data produced by other models:
\begin{align}
\label{eq:repeated_retraining}
    \theta_{t+1} = \RRM(\theta_t)= \argmin_{\theta \in \Theta} \DPR{\theta_t}{\theta} \quad    
\end{align}
They analyzed the behavior of this retraining dynamic under the following set of assumptions:

\begin{itemize}
    \item  The loss $\ell(z;\theta)$ is differentiable and $\gamma$-strongly convex in $\theta$. That is, for every $z \in \cZ$,
    \begin{equation}
    \tag{A1}
    \label{def:sc}
       \ell(z;\theta')  \geq \ell(z;\theta) + \langle \nabla \ell(z;\theta), \theta' -\theta \rangle + \frac{1}{2} \gamma \cdot  \|\theta- \theta'\|_2^2 
    \end{equation}
    \item The loss $\ell(z;\theta)$ is $\beta$-smooth in $z$. For every $z,z' \in \cZ$,
    \begin{equation}
   	\tag{A2}
    \label{eq:beta-smooth}
        \norm{\nabla_{\theta}\ell(z;\theta)-\nabla_{\theta}\ell(z';\theta)}
        \leq\beta \cdot \|z-z'\|_2\qquad\text{for all }z,z'\in\cZ.
    \end{equation}
    That is, the gradient of $\ell(z;\theta)$ with respect to $\theta$ is $\beta$-Lipschitz in $z$.\footnote{Throughout our work, all gradients of the loss function are with respect to the model argument $\theta$. We therefore write $\nabla \ell(z;\theta)$ as shorthand for $\nabla_\theta \ell(z;\theta)$. \label{footnote:gradient_notation}}
    \item The distribution map $\cD(\cdot)$ is $L$-Lipschitz. For every $\theta$ and $\theta'$ in $\Theta$
    \begin{equation}
    \tag{A3}
    \label{def:lipchitz_dist}
        W_1(\cD(\theta),\cD(\theta'))\leq L \cdot \| \theta-\theta' \|_2
        \qquad\text{for all }\theta,\theta'\in\Theta.
    \end{equation}
  
\end{itemize}
Formally, this Lipschitz condition on $\cD(\cdot)$ imposes that for every pair of models $\theta,\theta'$, the earthmover's, or Wasserstein-1, distance between distributions induced by the models is bounded by a constant $L$ times the $\ell_2$ distance in parameters. Intuitively, it formalizes the idea that small changes in the predictor only lead to small changes in the induced distribution. The other two assumptions hold for common losses in machine learning such as (regularized) log loss or squared loss.

 PZMH proved that retraining as defined in \Cref{eq:repeated_retraining} satisfies the following inequality:
\begin{align*}
    \|\RRM(\theta) - \RRM(\theta')\|_2 \leq \rho \cdot \|\theta- \theta'\|_2.
\end{align*}
Hence, repeatedly retraining converges at a linear rate to a unique performatively stable model if $\rho < 1$. \citet{mendler2020stochastic} and \citet{drusvyatskiy2023stochastic} built on this analysis to show that, in this same $\rho < 1$ regime, classical stochastic optimization procedures, such as gradient descent, converged to the same performatively stable model in finite samples\footnote{RRM is an idealized algorithm that assumes exact access to $\cD(\cdot)$.}.

In their work, PZMH also proved that the set of conditions above -- strong convexity, smoothness, and Lipschitz $\cD(\cdot)$ -- were necessary. If you remove any one of the assumptions, retraining as defined in \Cref{eq:repeated_retraining} can diverge. This was disappointing since the continuity assumption on $\cD(\cdot)$ is provably false in many practical applications. Whenever decision-makers take actions by thresholding a predictor, $\cD(\cdot)$ becomes discontinuous and stable points might fail to exist (see Example 3.4 in PZMH). 

This thresholding happens, for instance, in medicine, whenever physicians intervene on patients whose probability of disease exceeds a fixed threshold. The decision to treat leads to a discontinuous jump in the likelihood of the outcome (see e.g. \citet{chyn2021returns,dews,mendler2024engine} for empirical examples of this behavior in medicine, education, and online search).  The prior analysis also fails to capture settings where $\cD(\cdot)$ is continuous, but performative effects are strong, meaning that $\rho \gg 1$. While stable points are guaranteed to exist for any finite value of $\rho$, \citet{anagnostides2026computational} proved that finding stable points in this $\rho$ bounded away from $1$ regime is PPAD-hard. 

Recently, \citet{farina2026stability} showed that if one runs any no-regret algorithm on the sequence of stochastic losses $\ell_t(\theta) = \ell(z_t;\theta)$ where $z_t \sim \cD(\theta_t)$ then the uniform distribution over the iterates $\theta_1, \dots, \theta_T$ is $\mathrm{Regret}_T /T$-performatively stable regardless of the underlying distribution map $\cD(\cdot)$. This result proved that many natural continual learning algorithms like online gradient descent converge to a natural, randomized notion of stability, but they required the learner to deploy $\poly(1/\eps)$ many models to find an $\eps$-performatively stable mixture.

\paragraph{Our Techniques. }The conceptual starting point for our paper comes from observing that, while the existence of stable mixtures was already known as an implication of no-regret algorithms, one can directly prove existence via the minimax theorem. In particular, let $\mu$ be any distribution over the space of models $\Theta$ and let $\theta'$ be a fixed comparator. Define,
\begin{equation}
\label{eq:inf_stability_gap}
    \Phi(\mu, \theta') \coloneqq \E_{\theta\sim\mu} [\DPR{\theta}{\theta}- \DPR{\theta}{\theta'}].
\end{equation}
Observe that $\Phi(\mu, \theta')$ is affine in $\mu$ and concave in $\theta'$ for losses $\ell$ that are convex in $\theta'$. By definition, a mixture $\muPS$ is performatively stable if 
\begin{align*}
    \E_{\theta\sim \muPS} \left[\DPR{\theta}{\theta}\right] - \min_{\theta' \in \Theta}  \E_{\theta\sim \muPS} \left[\DPR{\theta}{\theta'}\right]  = \max_{\theta' \in \Theta} \Phi(\muPS, \theta') \leq 0.
\end{align*}
 Since for any $\theta'$, setting $\mu_{\theta'}$ to be a point mass at $\theta'$ yields $\Phi(\mu_{\theta'}, \theta') =0$, we have that $$\max_{\theta'} \min_{\mu} \Phi(\mu, \theta') \leq 0.$$ Applying the minimax theorem, we can swap the min and max, concluding that, $$\min_{\mu} \max_{\theta'}  \Phi(\mu, \theta') \leq 0,$$ and that a stable mixture must therefore exist.  

However, more than just showing existence, recent work by \citet{sigecom} shows that one can turn existence proofs for solutions to minimax problems into algorithms. In particular, since $\ell$ is convex, we can linearize the expression in \Cref{eq:inf_stability_gap}, 
\begin{align*}
    \E_{\theta\sim\mu} [\DPR{\theta}{\theta}- \DPR{\theta}{\theta'}] &\leq \E_{\theta \sim \mu} \E_{z \sim \cD(\theta)} \left[\nabla \ell(z;\theta)^\top(\theta - \theta') \right] \\
    & = \E_{\theta \sim \mu} S(\theta)^\top(\theta - \theta'),
\end{align*}
where $S(\theta) \coloneqq \E_{z\sim \cD(\theta)} \nabla \ell(z;\theta)$ is a function from $\R^d$ to $\R^d$. 

This rewriting converts the problem of computing a performatively stable mixture into solving an expected variational inequality, or EVI, where given an operator $S: \Theta \rightarrow \R^d$ the goal is to find a $\mu$ such that for every $\theta' \in \Theta$
\begin{align}
\label{eq:informal_EVI}
    \E_{\theta \sim \mu} S(\theta)^\top(\theta - \theta') \leq 0.
\end{align}
The study of these nonlinear optimization problems was recently initiated by \citet{EVI} who showed that they can be solved efficiently for any bounded operator $S$ given access to an evaluation oracle that given $\theta$ returns $S(\theta)$.

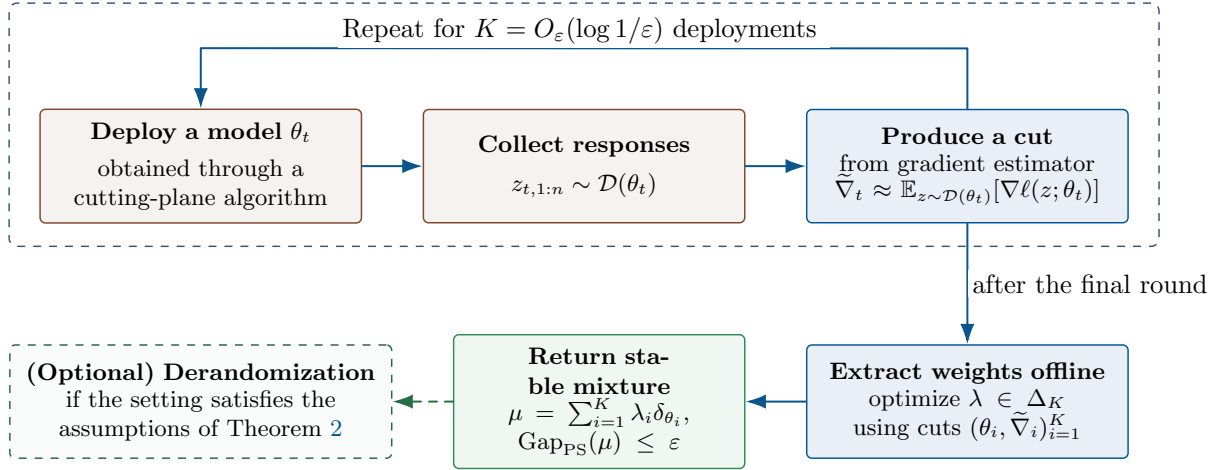
\begin{figure}[t]
\centering
\begin{tikzpicture}[
  node distance=9mm and 8mm,
  every node/.style={font=\footnotesize,align=center},
  stage/.style={rounded corners=2pt,line width=.55pt,
    inner xsep=5pt,inner ysep=5pt,minimum height=15mm,text width=39mm},
  deploy/.style={stage,draw=Sepia,fill=Sepia!7},
  update/.style={stage,draw=DarkBlue,fill=DarkBlue!6},
  output/.style={stage,draw=darkGreen,fill=darkGreen!7},
  optional/.style={stage,draw=darkGreen!75!black,
    fill=darkGreen!3,dashed,text width=47mm},
  loop boundary/.style={draw=DarkBlue!60!black,dashed,rounded corners=3pt,
    line width=.55pt,inner xsep=4mm,inner ysep=3mm},
  flow/.style={-{Latex[length=2.7mm,width=1.9mm]},
    line width=.75pt,line cap=round,line join=round,
    draw=DarkBlue!85!black},
  optional flow/.style={flow,dashed,draw=darkGreen!80!black},
  edge label/.style={
    font=\small,fill=white,inner sep=1.5pt,text=black}
]
\node[deploy] (deploy) {
  \textbf{Deploy a model \(\theta_t\)}\\[1mm]
  obtained through a cutting-plane algorithm
};
\node[deploy,right=of deploy] (responses) {
  \textbf{Collect responses}\\[-15pt]
  \[z_{t,1:n}\sim\cD(\theta_t)\]
};
\node[update,right=of responses] (update) {
  \textbf{Produce a cut}\\[-1pt]
  from gradient estimator\\[-1pt]
  $\widetilde\nabla_t\approx
  \E_{z\sim\cD(\theta_t)}[\nabla\ell(z;\theta_t)]$
};

\coordinate[above=11mm of responses] (looptop);
\node[loop boundary,fit=(deploy)(responses)(update)(looptop)]
  (adaptive-loop) {};

\draw[flow] (deploy) -- (responses);
\draw[flow] (responses) -- (update);
\coordinate[above=8mm of update.north] (loop-right);
\coordinate[above=8mm of deploy.north] (loop-left);
\draw[flow] (update.north) -- (loop-right)
  -- node[edge label,above] {
    Repeat for $K = O_\eps(\log 1/\eps)$ deployments
  } (loop-left) -- (deploy.north);

\node[update,below=16mm of update] (weights) {
  \textbf{Extract weights offline}\\[-1pt]
  optimize $\lambda\in\Delta_K$ using cuts
  $(\theta_i,\widetilde\nabla_i)_{i=1}^K$
};
\node[output,left=of weights,text width=35mm] (mixture) {
  \textbf{Return stable mixture}\\[-1pt]
  $\mu=\sum_{i=1}^K\lambda_i\delta_{\theta_i}$,
  $\operatorname{Gap}_{\rm PS}(\mu)\leq\eps$
};
\node[optional,left=of mixture] (mean) {
  \textbf{(Optional) Derandomization}\\[-1pt]
  if the setting satisfies the\\assumptions of \cref{thm:stable-mean}
};

\draw[flow] (update) --
  node[edge label,right] {after the final round} (weights);
\draw[flow] (weights) -- (mixture);
\draw[optional flow] (mixture) -- (mean);
\end{tikzpicture}
\caption{Pictorial depiction of our algorithm. The algorithm alternates between deployments and cutting-plane
updates, then extracts the mixture weights offline.  Under the additional
hypotheses of \cref{thm:stable-mean}, the output mixture can be replaced by
its mean.}
\label{fig:overview}
\end{figure}

In particular, they build on a rich literature \cite{papadimitriou2008computing,daskalakis2025efficient} on ellipsoid-based algorithms for computing coarse correlated equilibria to develop a procedure that can compute an $\eps$-approximate solution to the EVI in \Cref{eq:informal_EVI} using $\poly(d, \log(1/\eps))$ evaluations of $S$. 

These procedures query the operator $S(\cdot)$ at various points to repeatedly bisect the solution space $\Theta$ to find a small discrete set of ``approximate best responses'' $\cV$ which satisfy the property that for any $\theta' \in \Theta$ there exists at least one $\theta \in \cV$ such that $S(\theta)^\top(\theta-\theta') \leq \eps$. Viewing the problem as a zero-sum game, one can then use duality and linear programming to compute weights over $\cV$ to find a distribution $\mu$ over the elements in $\cV$ that is good for every $\theta' \in \Theta$.

Our work contributes to this methodology in two ways in order to compute performatively stable points with minimal model deployments. First, in this performative prediction setting, the operator $S(\theta) = \E_{z\sim \cD(\theta)} \nabla \ell(z;\theta)$ is unknown since $\cD(\cdot)$ is unknown.
We therefore provide a refined analysis of this Ellipsoid-Against-Hope machinery so that the algorithm works given only approximate access to $S(\cdot)$. Second, as indicated by the name, this methodology relied on the ellipsoid algorithm which performs $\Omega(d^2 \log(1/\eps))$ queries. In this paper, we refactor out the different steps of the analysis and replace the ellipsoid with modern cutting-plane algorithms that succeed with $\widetilde{\cO}(d\log(1/\eps))$ (approximate) queries of $S(\cdot)$. See \Cref{fig:overview} for a conceptual overview of this algorithmic approach.

The improvement in the dimension dependence is particularly valuable in the context of our work since querying $S(\cdot)$ amounts to deploying a new model which as we discussed can be very costly. Furthermore, it turns out to be optimal up to polylog factors in $d$. We can embed known hard instances from convex optimization into the performative prediction framework to show that $\Omega(d \log(1/\eps))$ deployments are necessary for any (possibly randomized) learning algorithm to find stable mixtures when $\eps\leq d^{-1}$.


The key part of the analysis is to understand how many samples $n_k$ are necessary to collect at every step $k$ and how to balance the errors
$\alpha_k \coloneqq \|\widehat{\nabla}_k-S(\theta_k)\|_2
=\|\widehat{\nabla}_k-\E_{z\sim \cD(\theta_k)}\nabla \ell(z;\theta_k)\|_2$
between the population and empirical operator evaluations. We show that for
general convex functions, it suffices to set $\alpha_k \propto \eps$, so
dimension-free concentration of the Euclidean norm permits
$n_k=\widetilde O(G^2R^2/\eps^2)$ and gives an $\eps$-EVI-stable mixture. 
Here, $G$ is a norm bound on the gradients $\nabla \ell(z;\theta)$ and $R$ is a norm bound on the parameters $\theta$.
For strongly convex functions, the added curvature acts as a margin condition. This enables us to set $\alpha_k\propto\sqrt{\eps}$
and $n_k=\widetilde O(G^2/(\gamma\eps))$ for performative stability. 

\subsection{Further Related Work and Discussion} 
\label{sec:discussion}
\paragraph{Stateful and Multiplayer Performative Prediction.} Our work focuses on the initial performative prediction setup where the learner deploys a model $\theta$ and observes samples $z$ drawn from $\cD(\theta)$. In some applications, however, the observed data might depend on the predictions made by multiple actors. This is known as the multiplayer version of performative prediction studied in \citet{narang2023multiplayer,piliouras2023multi,li2022multi}.
The distribution can also depend on the entire history of models. This stateful performative prediction problem was introduced in \citet{brown2022performative} and extended in \citet{lee2026partially} to account for both endogenous and exogenous shifts. Very recently, Hardt provides a refined analysis of retraining in stateful settings showing how repeatedly retraining converges to stable signals in the distribution under a specific decomposition of the distribution map $\cD(\cdot)$ into performative and non-performative components \cite{hardt2026retraining,hardt2026future}.
We believe it is an interesting direction for future work to understand whether the EVI and cutting-plane machinery developed herein can be extended to these other more complex settings to minimize the number of model deployments. 

\paragraph{Performative Optimality.} There is a parallel line of work in performative prediction \cite{miller2021outside,kim2022making,izzo2021learn,jagadeesan2022regret,pluginpp} that aims to find performatively \textit{optimal} points. A model $\thetaPO$ is performatively optimal if it directly minimizes the performative risk:
\begin{align*}
    \thetaPO \in \argmin_{\theta \in \Theta} \E_{z\sim \cD(\theta)} \ell(z;\theta).
\end{align*}
Stability and optimality are in general distinct concepts. While we now have a relatively mature set of algorithmic tools to compute stable points without making any assumptions on how predictions respond to distributions \cite{farina2026stability}, the landscape of approaches for finding performatively optimal points is relatively scarce. Furthermore, all known algorithms make significant simplifying assumptions on $\cD(\cdot)$. For instance, \citet{miller2021outside} assume a location-scale family for $\cD(\cdot)$ and \citet{kim2022making} assume that predictions only influence outcomes and not features. 

While stable models need not be optimal and optimal models need not be stable for worst case distribution maps $\cD(\cdot)$, we believe that it is an interesting direction for future work to identify natural conditions on the distribution map $\cD(\cdot)$ and the loss $\ell$ under which these two solution concepts coincide. This would allow us to port over the tools developed for computing stable models to this so far more challenging solution concept.

\paragraph{Joint Deployment and Sample Complexity.} Our procedure computes stable mixtures using only $\widetilde{\cO}(d\log(1/\eps))$ deployments, which is optimal up to logarithmic factors in $d$ when $\eps\leq d^{-1}$. A dimension-free concentration argument for bounded Euclidean gradients shows that each deployment needs only $\widetilde O(1/\eps^2)$ samples, and hence the same procedure uses $\widetilde O(d/\eps^2)$ samples in total. 
For strongly convex losses, performative stability requires only $\widetilde O(1/\eps)$ samples per deployment and $\widetilde O(d/\eps)$ samples in total. While our procedure minimizes the number of deployments, we do not know whether it is optimal in terms of the total number of samples used. We believe it is an interesting direction for future work to understand what is the necessary sample complexity for performative stability and if there are any tradeoffs between sample and deployment complexity.

\section{Structural Insights on Performative Stability}
\label{sec:structural}

Throughout, let $\Theta$ be a nonempty compact convex subset of a Euclidean
space. Following PZMH (Definition~2.3), we define the \emph{decoupled performative risk},
\begin{align}
\label{eq:DPR_definition}
\tDPR(\theta, \theta') \coloneqq \DPR{\theta}{\theta'}.
\end{align}
As an analysis tool, the decoupled performative risk measures the loss of the model $\theta'$ over the distribution induced by $\theta$.
By definition, given a distribution $\mu$ over $\Theta$,
the extent to which it is performatively stable is captured by the gap\footnote{Technically, we would need to also require that $\mu$ is such that the expectation $\E_{\theta\sim\mu}$ is defined, but we will not need to worry about this in the sequel. All distributions we will generate will be finitely supported, and hence the expectation is always defined. Throughout the paper we will also assume that the loss functions are continuous and the expectation and differentiation can be interchanged.}
\begin{equation}
    \operatorname{Gap}_{\rm PS}(\mu)
    \coloneqq
    \sup_{\theta'\in\Theta}
    \E_{\theta\sim\mu}
    [\tDPR(\theta, \theta) -\tDPR(\theta, \theta')].
    \label{eq:ps-gap}
\end{equation}

\subsection{Performatively Stable Mixtures from Minimax Duality}

We now prove the existence of stable mixtures as a consequence of Sion's minimax theorem. The existence result of \citet{farina2026stability} can also be viewed as an indirect consequence of the following result by leveraging the folklore constructive proof of the minimax theorem by \citet{freund1999adaptive} together with a concentration argument. The correctness of our cutting-plane algorithm also follows from this connection to minimax duality. It, however, uses a faster constructive proof of the minimax theorem due to \citet{sigecom,farina2024polyhedral}.

\begin{theorem}[Approximate existence of stable mixtures]
\label{thm:approximate-existence}
Suppose that the loss $\ell(z;\theta)$ is continuous and convex in $\theta$ and let $\Delta(\Theta)$ denote the convex set of finitely supported distributions on $\Theta$. Then,
\[
    \inf_{\mu\in\Delta(\Theta)}
    \operatorname{Gap}_{\rm PS}(\mu)
    \leq 0,
\]
that is, for every $\alpha>0$, there is a distribution $\mu$ supported on finitely many models in $\Theta$ that is $\alpha$-performatively stable:
\[
    \operatorname{Gap}_{\rm PS}(\mu)<\alpha.
\]
\end{theorem}

\begin{proof}
Consider the function
\[
    \Phi : \Delta(\Theta)\times\Theta\to\R,
    \qquad\quad
    \Phi(\mu,\theta')
    \coloneqq
    \E_{\theta\sim\mu}
    [\tDPR(\theta, \theta)-\tDPR(\theta, \theta')].
\]
For every $\mu\in\Delta(\Theta)$, the map $\Phi(\mu,\cdot)$ is
continuous and concave on the compact convex set $\Theta$ since $\ell$ is continuous and convex in $\theta$. For every
$\theta' \in\Theta$, the map $\Phi(\cdot,\theta')$ is affine. Therefore, Sion's minimax theorem
\cite{sion1958general} applies to $\Phi$. Recognizing that 
 $\operatorname{Gap}_{\rm PS}(\mu) = \sup_{\theta\in\Theta}\Phi(\mu,\theta)$,
we can therefore write
\[
    \inf_{\mu\in\Delta(\Theta)} \operatorname{Gap}_{\rm PS}(\mu) = \inf_{\mu\in\Delta(\Theta)}
    \sup_{\theta\in\Theta}\Phi(\mu,\theta)
    =
    \sup_{\theta\in\Theta}
    \inf_{\mu\in\Delta(\Theta)}\Phi(\mu,\theta).
\]
Now, consider the inner minimization on the right-hand side. For each fixed $\theta\in\Theta$, letting $\delta_\theta$ denote the point mass at $\theta$, we have
\[
    \inf_{\mu\in\Delta(\Theta)}\Phi(\mu,\theta) \le \Phi(\delta_\theta,\theta) = \tDPR(\theta, \theta) - \tDPR(\theta, \theta) =0.
\]
Consequently, $\sup_{\theta'\in\Theta}\inf_{\mu\in\Delta(\Theta)}\Phi(\mu,\theta') \le 0$, and hence
\(
    \inf_{\mu\in\Delta(\Theta)} \operatorname{Gap}_{\rm PS}(\mu) \leq 0
\)
as claimed.
\end{proof}

We remark that the result in \Cref{thm:approximate-existence} is tight. The definition of performative stability cannot be further strengthened in any of the following senses. 

First, the infimum cannot generally be promoted to a minimum. That is, it might not be possible to find a distribution $\mu$ with $\operatorname{Gap}_{\rm PS}(\mu)=0$. This is shown in the following proposition, which can be seen as the performative-risk version of the discontinuous sign
construction in Proposition~3.3 and Example~C.2 of \citet{EVI}.

\begin{proposition}[Exact stability might not be attained]
\label{prop:no-exact-mixture}
There exists an instance with $\Theta=[-1,1]$  and a bounded, convex, and continuous loss $\ell(z;\theta)$
that is affine in $\theta$ for which every distribution $\mu$ has strictly
positive stability gap $\mathrm{Gap}_\mathrm{PS}(\mu) > 0$, although the infimum of these gaps is zero.
\end{proposition}

The construction takes $\cD(\theta)$ to be a point mass at the discontinuous
sign of $\theta$, with $\operatorname{sgn}(0)=1$, and uses the affine loss
$\ell(z;\theta')=1+z\theta'$. Its gap is
$\E|\theta|+|\E\operatorname{sgn}(\theta)|$: it is positive for every
distribution, but equals $a$ for the uniform mixture on $\{-a,a\}$. See \cref{app:structural-proofs}.

Furthermore, the assumption that the loss is convex is also necessary. Without convexity, even approximately stable performative \textit{mixtures} can fail to exist. There is a continuous and bounded nonconvex loss for which \textit{every} possible mixture has performative stability gap at least $1/4$.

\begin{proposition}[Convex Losses are Necessary for Stable Mixtures]
    \label{prop:nonconvex-no-mixture}
Assume $\Theta=[0,1]$. There exists a continuous and bounded loss function $\ell(z;\theta)$ that is nonconvex in $\theta$ such that every distribution $\mu$ over $\Theta$ has
performative stability gap $\mathrm{Gap}_\mathrm{PS}(\mu) \geq 1/4$.
\end{proposition}

The counterexample behind this proposition follows a simple pattern similar to prior counterexamples outlining the limits of the minimax theorem. The loss is a simple inverted quadratic $\ell(z; \theta)=1-(z-\theta)^2$ in one dimension. The reader can again find the details in \Cref{app:structural-proofs}.

\subsection{Strengthening Performatively Stable Mixtures}

In this section, we introduce a slightly stronger condition than performative stability, which we call \emph{performative EVI-stability} (or EVI-stability for short).  

This condition is a linearization of the performative stability condition from the previous section.
As one might expect from the name, performative EVI-stability is an instance of an expected variational inequality (EVI) problem, as introduced by \citet{EVI}. These problems have also appeared in the literature under the names of outgoing minimax problems \citep[Theorem~5]{foster2021forecast},  performance certificates \citep{nemirovski2010accuracy}, or negative correlation search problems \cite{perdomo2025defense}. 

We introduce this stronger notion since it exactly characterizes the kinds of mixtures computed by the cutting-plane procedure we introduce in the next section. 
As an added benefit, we show in \cref{sec:derandomization} that EVI-stable mixtures can be efficiently derandomized to obtain a single performatively stable model in well-behaved settings. 

\begin{definition}
Assume that $\ell(z;\theta)$ is differentiable in $\theta$ for every $z$. Given a distribution $\mu$ over $\Theta$, define 
\begin{align}
    \operatorname{Gap}_{\rm EVI}(\mu)
    &\coloneqq \max_{\theta'\in\Theta}\E_{\theta\sim\mu} \ip{\E_{z\sim\cD(\theta)}[\nabla_\theta\ell(z;\theta)]}{\theta-\theta'}.
    \label{eq:first-order-gap}
\end{align}
A mixture $\mu$ is $\alpha$-performatively EVI-stable if $ \operatorname{Gap}_{\rm EVI}(\mu) \leq \alpha$.
\end{definition}

Intuitively, a mixture is EVI-stable if the gradient of the loss evaluated over the induced distribution $\E_{z\sim \cD(\theta)} \nabla \ell(z;\theta)$ is uncorrelated with any direction in parameter space $v=\theta-\theta'$. While we prove the existence of EVI-stable mixtures constructively as a consequence of the cutting-plane procedure in the next section, one can prove existence of EVI-stable mixtures directly via a minimax argument as in \Cref{thm:approximate-existence}.

We start by relating EVI-stability to prior notions of performative stability. In particular, using the first-order condition of convexity, we can show that the EVI gap of a mixture always upper bounds its stability gap:

\begin{proposition}[EVI-stability implies stability]
\label{prop:evi-implies-stability}
Assume $\ell(z;\theta)$ is convex in $\theta$ for every $z$. Then, every $\alpha$-performatively EVI-stable distribution is
$\alpha$-performatively stable, that is,
\begin{equation*}
    \operatorname{Gap}_{\rm PS}(\mu)
    \leq \operatorname{Gap}_{\rm EVI}(\mu).
\end{equation*}
\end{proposition}

This relationship is the key fact we will use when establishing the correctness of our cutting-plane procedure. We remark that the inequality can be in fact strict for certain problems, highlighting how EVI-stability is in fact slightly \textit{stronger} than performative stability. 

While it is true that a model $\thetaPS$ is exactly performatively stable if and only if it is an exact (point mass mixture) solution to the EVI, approximately stable mixtures do not approximately solve the EVI. The first statement in the following proposition again follows directly from the first order optimality conditions for convex functions, and the second can be shown via a simple, one-dimensional counterexample (See \Cref{app:structural-proofs}). 

\begin{proposition}[Approximate Stability Does Not Imply Approximate EVI]
\label{prop:stability-not-evi}
Assume $\ell(z;\theta)$ is convex in $\theta$ for every $z$. For every $\theta\in\Theta$, the point mass $\delta_\theta$ is exactly performatively
stable if and only if it is exactly performatively EVI-stable, that is,
\[
    \operatorname{Gap}_{\rm PS}(\delta_\theta)=0
    \quad\Longleftrightarrow\quad
    \operatorname{Gap}_{\rm EVI}(\delta_\theta)=0.
\]
Nevertheless, there exist performative prediction instances over $\Theta=\cZ=[0,1]$, where the following statements hold even for bounded
losses $\ell(z;\theta)$ that are differentiable, convex, and $1$-Lipschitz in their model
argument $\theta$, and for which the distribution map $\cD(\cdot)$ is 1-Lipschitz:
\begin{enumerate}[label=\textup{(\roman*)}]
\item There is an exactly performatively stable mixture $\mu$ such that
\[
    \operatorname{Gap}_{\rm PS}(\mu)=0
    \qquad \text{and}\qquad
    \operatorname{Gap}_{\rm EVI}(\mu)=\frac14.
\]
\item There is another instance in which, for every
$\alpha\in(0,1/2]$, some $\theta_\alpha\in\Theta$ satisfies
\[
    \operatorname{Gap}_{\rm PS}(\delta_{\theta_\alpha})=\alpha,
    \qquad \text{and} \qquad
    \operatorname{Gap}_{\rm EVI}(\delta_{\theta_\alpha})\geq\frac{1}{2}.
\]
\end{enumerate}
\end{proposition}

\subsection{Derandomization of EVI-Stable Mixtures under Well-Conditioned Losses}\label{sec:derandomization}

Having defined EVI-stability and characterized its relationship to prior notions, in this last subsection we now show that the focus on mixtures is without loss of generality in the classical $\rho < 1$ regime that has been the focus of the bulk of prior work. We show that the mean of an EVI-stable mixture is itself performatively stable. 

The proof of the result follows directly from the definitions of strong convexity, performative EVI-stability, and the following sensitivity bound on the difference in gradients. In particular, using the variational definition of Wasserstein distance, one can show that if $\cD(\cdot)$ is $L$-Lipschitz  (\Cref{def:lipchitz_dist}) and $\ell$ is $\beta$-smooth in $z$ (\Cref{eq:beta-smooth}), then the following inequality is true for any $\theta, \theta', \bar{\theta} \in \Theta$,
\begin{equation}
\label{eq:performative-sensitivity}
 \left \| \E_{z \sim \cD(\theta)} \nabla \ell(z;\bar{\theta}) - \E_{z \sim \cD(\theta')} \nabla \ell(z;\bar{\theta})\right\|_2 \le L\beta \norm{\theta - \theta'}_2.
\end{equation}
See \citep[Lemma~D.4]{perdomo2020performative} for a proof of this statement. 

\begin{theorem}[Stability of the mean]
\label{thm:stable-mean}
Let $\Theta$ be compact and convex. Assume that $\ell$ is $\gamma$-strongly convex in $\theta$ \eqref{def:sc}, $\beta$-smooth in $z$ \eqref{eq:beta-smooth}, and that $\cD(\cdot)$ is $L$-Lipschitz \eqref{def:lipchitz_dist}. Let $\mu$ be an  $\alpha$ performatively-EVI-stable mixture
and write  
\[
    \bar\theta\coloneqq\E_{\theta\sim\mu}\theta \quad \text{ and } \quad      \theta^\star\in\argmin_{\theta'\in\Theta}\tDPR(\bar\theta,\theta').
\]
Then,
\begin{align}
    \DPR{\bar\theta}{\bar\theta} \leq 
    \min_{\theta \in \Theta}\DPR{\bar\theta}{\theta} +
 \alpha
    +\frac{L\beta-\gamma}{2}
       \left(
       \E_\mu\norm{\theta-\theta^\star}_2^2
       +\E_\mu\norm{\theta-\bar\theta}_2^2
       \right).
    \label{eq:mean-refined}
\end{align}
In particular, if $L \leq \gamma/\beta$ (i.e., $\rho\leq1$), then the mean $\bar\theta$ of $\mu$ is
$\alpha$-performatively stable:\footnote{Technically, the proof only requires that the map $\theta \mapsto \tDPR(\theta', \theta)$ is strongly convex in $\theta$ for every $\theta'$. This condition holds if $\ell$ is the squared loss and the marginal distribution over features is full-rank for every $\cD(\theta)$, even though the squared loss is not pointwise strongly convex. We state the condition on the loss directly for simplicity.}
\begin{align*}
     \DPR{\bar\theta}{\bar\theta} \leq 
    \min_{\theta \in \Theta}\DPR{\bar\theta}{\theta} +
 \alpha.
\end{align*}
\end{theorem}

\begin{proof}
Compactness and strong convexity guarantee the existence and uniqueness of $\theta^\star$, while convexity of
$\Theta$ ensures that $\bar\theta\in\Theta$. Strong convexity, applied with the
gradient at $\theta$, gives
\begin{equation}
    \DPR{\bar\theta}{\theta}-\DPR{\bar\theta}{\theta^\star}
    \leq
    \ip{\E_{z\sim \cD(\bar{\theta})} \nabla \ell(z;\theta)}{\theta-\theta^\star}
    -\frac\gamma2\norm{\theta-\theta^\star}_2^2.
    \label{eq:mean-proof-strong-convexity}
\end{equation}
Add and subtract
$\E_{z\sim\cD(\theta)}[\nabla\ell(z;\theta)]$. By
\eqref{eq:performative-sensitivity}, Cauchy-Schwarz, and
$ab\leq(a^2+b^2)/2$,
\begin{align*}
    \ip{\E_{z\sim \cD(\bar{\theta})} \nabla \ell(z;\theta) -\E_{z\sim \cD(\theta)} \nabla \ell(z;\theta)}{\theta-\theta^\star}
    &\leq
    L\beta\norm{\theta-\bar\theta}_2 \cdot \norm{\theta-\theta^\star}_2\\
    &\leq
    \frac{L\beta}{2}\norm{\theta-\bar\theta}_2^2
    +\frac{L\beta}{2}\norm{\theta-\theta^\star}_2^2.
\end{align*}
Averaging \eqref{eq:mean-proof-strong-convexity} and using performative
EVI-stability with comparator $\theta^\star$ yields
\begin{align}
    \E_\mu[\DPR{\bar\theta}{\theta}-\DPR{\bar\theta}{\theta^\star}]
    \leq \alpha
    +\frac{L\beta-\gamma}{2}\E_\mu\norm{\theta-\theta^\star}_2^2
    +\frac{L\beta}{2}\E_\mu\norm{\theta-\bar\theta}_2^2.
    \label{eq:mean-proof-upper}
\end{align}
It remains to replace $\E_\mu[\DPR{\bar\theta}{\theta}]$ by $\DPR{\bar\theta}{\bar\theta}$. Strong convexity at $\bar\theta$
gives
\[
    \DPR{\bar\theta}{\theta}
    \geq \DPR{\bar\theta}{\bar\theta}
    +\ip{\E_{z\sim \cD(\bar{\theta})} \nabla \ell(z;\bar{\theta})}{\theta-\bar\theta}
    +\frac\gamma2\norm{\theta-\bar\theta}_2^2.
\]
The linear term vanishes after taking expectations, as $\bar\theta=\E_\mu\theta$ by definition. Substituting
the resulting lower bound into \eqref{eq:mean-proof-upper} proves
\eqref{eq:mean-refined}. Finally, if $L\beta\leq\gamma$, the fraction $(L\beta-\gamma)/2$ on the right of \eqref{eq:mean-refined} is nonpositive and multiplies a nonnegative parenthesis, leaving a stability gap of at most $\alpha$ and concluding the proof.
\end{proof}

\Cref{thm:stable-mean} also implies that when $\rho < 1$, the mean of an EVI-stable mixture is not only performatively stable, but also close to the (unique) performatively stable model. Indeed, \eqref{eq:mean-refined} yields
\[
\E_\mu\norm{\theta-\theta^\star}_2^2 +    \E_\mu\norm{\theta-\bar\theta}_2^2
    \leq
    \frac{2\alpha}{\gamma - L\beta}.
\]
Dropping $\E_\mu\norm{\theta-\bar\theta}^2$ and using Jensen's inequality, one therefore has
\[
   \norm{\bar{\theta}-\theta^\star}_2^2 = \norm{\E_\mu \theta - \theta^\star}_2^2 \le \E_\mu \norm{\theta - \theta^\star}_2^2 \le \frac{2\alpha}{\gamma - L\beta}.
\]
Recall that $\theta^\star = \RRM(\bar\theta)$. Moreover, when $\rho<1$, the map $\RRM(\cdot)$ is a contraction and has a unique fixed point $\thetaPS$. Therefore,
\begin{align*}
    \norm{\bar\theta-\thetaPS}_2
    &\leq
    \norm{\bar\theta-\RRM(\bar\theta)}_2
    +
    \norm{\RRM(\bar\theta)-\RRM(\thetaPS)}_2\\
    &\leq
    \norm{\bar\theta-\theta^\star}_2
    +
    \rho\norm{\bar\theta-\thetaPS}_2.
\end{align*}
Rearranging and combining with the preceding inequality gives
\[
    \norm{\bar\theta-\thetaPS}_2^2
    \leq
    \frac{1}{(1-\rho)^2}
    \norm{\bar\theta-\theta^\star}_2^2
    \leq
    \frac{2\alpha}{(\gamma-L\beta)(1-\rho)^2}.
\]
Thus, whenever $\rho<1$, the mean of an $\alpha$-EVI-stable mixture is also $\cO(\alpha)$ close in squared norm to the unique performatively stable model. At the boundary case when $\rho=1$, the mean remains $\alpha$-performatively stable by \Cref{thm:stable-mean}, but the distance guarantee above becomes vacuous.

On a technical note, notice that since $\ell$ is strongly convex, $\tDPR(\theta, \theta') = \E_{z\sim \cD(\theta)} \ell(z;\theta')$ is convex in $\theta'$ and one can intuitively try to apply Jensen's inequality to move from the mixture to its mean. The challenge in the proof of the theorem above is addressing the nonconvex dependence on $\theta$ in the first argument  as it appears inside the distribution map $\cD(\cdot)$. To do this, we make use of the gradient comparison inequality in \Cref{eq:performative-sensitivity} (a fact which does not follow from convexity at all) as well as the guarantee that the mixture is EVI-stable.

We remark that, together with the algorithm in the next section, this structural result yields the first known procedure for computing performatively stable points in finite samples at the critical $\rho=1$ threshold. It builds on and extends the work of \citet{anagnostides2026computational} who showed how to compute deterministically stable models when $\rho =1$ in an idealized setup where the learner knew $\cD(\cdot)$ and could exactly compute $\mathrm{RRM}(\theta)$ for any $\theta$.
In light of prior work in this area showing how stable models are easily computable via gradient descent for $\rho <1$ and PPAD-hard to find for $\rho \gg 1$, this last result establishes a complete picture regarding the computational complexity of finding stable points in finite samples in the standard performative prediction setup.

\section{Deployment-Efficient Algorithms for Finding Stable Mixtures}
\label{sec:few-deployments}

In this section, we show that one can always find an $\eps$-EVI-stable mixture (and hence, an $\eps$-performatively stable mixture by \Cref{prop:evi-implies-stability}) using only $\widetilde O(d\log(1/\eps))$ deployments, as opposed to the $\poly(1/\eps)$ deployments required by prior methods. 

At a technical level, our result follows from connecting performativity with recent fast computational minimax constructions \citep{sigecom,EVI}, together with a concentration argument, to produce a fast cutting-plane method. We describe the algorithm in \cref{sec:algorithm-overview}, its guarantees in \cref{sec:algo-guarantees}, and the analysis in \cref{sec:algo-analysis}.

\subsection{Algorithm Overview}
\label{sec:algorithm-overview}

Our construction operates in the standard performative prediction setup,
where the distribution map $\cD(\cdot)$ is fixed but unknown. The algorithm is summarized
pictorially in \cref{fig:overview} in the introduction and in pseudocode in
\cref{alg:few-deployment}. At every time $t$, our method proposes a model
$\theta_t$ using a cutting-plane method. Then, we deploy it and collect a batch of samples drawn from $\cD(\theta_t)$. The responses are used to estimate the population EVI operator, which is used to produce a new cut that is fed back to the cutting-plane method, and the process repeats.

More specifically, assume that $\Theta\subseteq\R^d$ is a well-bounded
convex body \citep[Definition~2.1.16]{groetschel1988geometric}. That is,
$\Theta$ is convex and compact, and there are known radii $0<r\leq R$ such
that\footnote{$B_2(c,r)$ denotes an $\ell_2$ ball of radius $r$ around a (center) point $c\in \R^d$.}
\begin{equation}
    B_2(c,r)\subseteq\Theta\subseteq B_2(0,R)
    \quad\text{for some }c\in\R^d.
    \label{eq:rounding}
    \tag{A4}
\end{equation}
We assume access to an exact separation oracle for
$\Theta$. For a well-bounded convex body, the standard polynomial-time
separation--optimization equivalence \citep{groetschel1988geometric} gives a linear-optimization oracle over
$\Theta$ to any prescribed additive accuracy; we suppress the resulting
polynomial arithmetic and oracle overhead.
When a model $\theta$ is deployed, each observed response
$z\sim\cD(\theta)$ provides the sample gradient
$\nabla\ell(z;\theta)$.  (We recall from
\Cref{footnote:gradient_notation} that all gradients of the loss function are
with respect to the model argument $\theta$.)
We assume that the sample gradients are
uniformly bounded:
\begin{equation}
    \norm{\nabla \ell(z;\theta)}_2\leq G
    \qquad
    \text{for every }(z,\theta)\in\cZ\times\Theta,
    \label{eq:sample-gradient-bound}
    \tag{A5}
\end{equation}
for a known $G\geq0$. For the asymptotic bounds below, we work in the
nontrivial regime $GR\geq2\eps$. This entails no loss because $G$ is an
upper bound and may be enlarged to $\max\{G,2\eps/R\}$.

\paragraph{Phase I: Deployment.}
We initialize the volumetric cutting-plane method of
\citet{jiang2020improved} on the box $Q_0=[-R,R]^d$, which contains
$\Theta$ by \eqref{eq:rounding}.  At each cutting-plane round, the method proposes a vector of parameters $\tilde{\theta} \in \R^d$, which might or might not belong to $\Theta$.
If $\tilde{\theta}\notin\Theta$, the separation oracle will return $a$ such that
$\ip{a}{\theta' - \tilde{\theta}} < 0$ for every $\theta'\in\Theta$.
In this case, we skip deploying $\tilde{\theta}$, cut the active search region with the halfspace
$H=\{\theta':\ip{a}{\theta' - \tilde{\theta}}\leq 0\}$, and continue the cutting-plane method.

If $\tilde{\theta}\in\Theta$, we instead make a deployment.  Indexing deployments by
$t$, set $\theta_t=\tilde{\theta}$ and collect $n$ fresh responses
$z_{t,1},\ldots,z_{t,n} \sim \cD(\theta_t)$.  We average the sample gradients to obtain
\begin{equation}
    \widetilde\nabla_t
    \coloneqq\frac1n\sum_{j=1}^n \nabla \ell(z_{t,j};\theta_t). 
    \label{eq:empirical-gradient}
\end{equation}

Provided $n$ is set large enough, $\widetilde\nabla_t$ is a good approximation of the population EVI operator
$S(\theta_t)\coloneqq\mathbb{E}_{z\sim\cD(\theta_t)}[\nabla \ell(z;\theta_t)]$.
Here the distribution is held fixed at the deployed model; no derivative of
$\cD(\theta)$ is taken. If $\widetilde\nabla_t\neq0$, we supply the cut
\[
    H=\left\{\theta':
      \ip{\widetilde\nabla_t}{\theta' -\theta_t}\leq0\right\}
\]
to the cutting-plane method. In the unlikely event that $\widetilde\nabla_t=0$, we can stop the cutting-plane method early because this zero normal already gives an exact empirical EVI certificate. The population guarantee still follows from the concentration-and-transfer argument below.

Overall, for target accuracy $\eps$, one can bound (see \cref{app:volumetric-localization}) the number of rounds of the cutting-plane method to at most
\[
    T=O\left(d\log\left(
      \frac{dGR^2}{r\eps}
    \right)\right).
\]

\begin{algorithm}[t]
    \DontPrintSemicolon
\caption{Stable mixture from few deployments}
\label{alg:few-deployment}
\KwIn{Accuracy $\eps>0$, batch size $n\geq1$, round cap $T$ as above, and cutting-plane method $\mathsf{CP}$}
\KwOut{A finitely supported mixture $\mu$ on $\Theta$}
Initialize the volumetric cutting-plane method on
$Q_0=[-R,R]^d$; set $K\gets0$ and initialize $\cV$ as an empty list\;
\For{$i=1,\ldots,T$}{
    $\tilde{\theta}\gets\fnm{Query}(\mathsf{CP})$\;
    \eIf{$\tilde{\theta}\notin\Theta$}{
        Obtain $a\neq0$ with $\ip{a}{\theta'-\tilde{\theta}}<0$ for every
        $\theta'\in\Theta$ from the separation oracle for $\Theta$\;
        $H\gets\{\theta':\ip{a}{\theta'-\tilde{\theta}}\leq0\}$\;
    }{
        $K\gets K+1$; set $\theta_K=\tilde{\theta}$, deploy it, and observe
        $z_{K,1},\ldots,z_{K,n}\stackrel{\rm iid}{\sim}\cD(\theta_K)$\;
        $\widetilde\nabla_K\gets
        n^{-1}\sum_{j=1}^n\nabla_\theta\ell(z_{K,j};\theta_K)$ and append
        $(\theta_K,\widetilde\nabla_K)$ to $\cV$\;
        \textbf{if} $\widetilde\nabla_K=0$ \textbf{then break}\;
        $H\gets\{\theta':\ip{\widetilde\nabla_K}{\theta'-\theta_K}
        \leq0\}$\;
    }
    $\fnm{Update}(\mathsf{CP},H)$\;
}
Set
$\lambda\gets\fnm{ExtractWeights}(\cV,\eps/3)$\;
\KwRet{$\mu=\sum_{i=1}^K\lambda_i\delta_{\theta_i}$}\;
\end{algorithm}

\paragraph{Phase II: Mixture Extraction.}
Given the transcript of deployments and corresponding gradient estimates
$\cV=\{(\theta_i,\widetilde\nabla_i)\}_{i=1}^K$ obtained by averaging samples, define
the empirical game value
\begin{equation}
    \widehat\Psi_{\cV}(\lambda, \theta')
    \coloneqq
    \sum_{i=1}^K\lambda_i\ip{\widetilde\nabla_i}{\theta_i-\theta'},
    \label{eq:empirical-game-value}
\end{equation}
for $\lambda \in \Delta_K$ and $\theta' \in \Theta$, where $\Delta_K=\{\lambda\in\R_+^K:\sum_{i=1}^K\lambda_i=1\}$ denotes the
probability simplex in $\R^K$. Note that the function is convex in $\lambda$.  We slightly abuse notation and introduce the shorthand
    \begin{equation}
    \widehat\Psi_{\cV}(\lambda)
    \coloneqq
    \max_{\theta'\in\Theta}
    \sum_{i=1}^K\lambda_i\ip{\widetilde\nabla_i}{\theta_i-\theta'}.
    \label{eq:empirical-game-value-max}
\end{equation}
Given the linear-optimization routine induced by separation over $\Theta$, its
value can be approximated and an $\eps$-subgradient obtained to arbitrary
accuracy $\eps>0$.

To extract the weights, we minimize $\widehat\Psi_{\cV}(\lambda)$ over the compact
convex set $\Delta_K$ to additive accuracy $\eps/3$. We denote this operation
by $\fnm{ExtractWeights}(\cV,\eps/3)$. Phase I will ensure that the optimal
value is at most $\eps/3$, so the returned weights have empirical game value
at most $2\eps/3$.
This computation can be performed offline in polynomial time using only the recorded
gradients and the separation oracle for $\Theta$, with no further
deployments. Finally, we return the atomic (performatively stable) distribution
$\mu=\sum_{i=1}^K\lambda_i\delta_{\theta_i}$, where $\delta_{\theta_i}$
is the point mass at $\theta_i$.

\subsection{Guarantees}
\label{sec:algo-guarantees}

The main result of this section is the following theorem, bounding the number of deployments and samples required to extract an $\eps$-performatively stable mixture.

\begin{theorem}[Finite-sample stabilization in few deployments]
\label{thm:finite-sample-master}
Assume $\Theta$ is a well-bounded convex body \eqref{eq:rounding} and that $\ell$ has gradients uniformly bounded by $G$ \eqref{eq:sample-gradient-bound}. Let
$\eps>0$ and $\delta\in(0,1)$. Run \cref{alg:few-deployment} with
\[
    T=O\left(d\log\left(
      \frac{dGR^2}{r\eps}
    \right)\right)
\]
as the cutting-plane round cap and $\mathsf{CP}$ set to the cutting-plane
method of \citet{jiang2020improved}. Then the following statements hold with
probability at least $1-\delta$:
\begin{enumerate}
    \item If $\ell$ is convex in $\theta$ for every $z$ and
    \[
        n\geq \Omega\left(
        \frac{G^2R^2}{\eps^2} \log(T/\delta) \right),
    \]
    then the output mixture is $\eps$-performatively EVI-stable and hence
    $\eps$-performatively stable.
    \item If $\ell$ is $\gamma$-strongly convex in $\theta$ for every $z$ and
    \[
        n\geq \Omega\left(
        \frac{G^2}{\gamma\eps} \log(T/\delta) \right),
    \]
    then the output mixture is $\eps$-performatively stable and also obeys
    \[
        \operatorname{Gap}_{\rm EVI}(\mu)
        \leq \cO(\sqrt{\eps}).
    \]
    If the batch size instead satisfies the bound in case~\textup{(i)}, the
    output is $\cO(\eps)$-performatively EVI-stable as well.
\end{enumerate}
In either case, the output has the form
$\mu=\sum_{i=1}^K\lambda_i\delta_{\theta_i}$ with $K\leq T$.
\end{theorem}

The distinction in the strongly convex case is essential. Strong convexity absorbs
estimation error quadratically in the performative-stability loss gap, but
the EVI transfer remains linear. Thus the smaller batch gives the displayed
$O(\sqrt\eps)$ EVI bound, while an $\eps$-EVI guarantee in either case uses
the batch size from the guarantee for convex functions.

This result, together with \Cref{thm:stable-mean}, complements and strengthens Theorem 3.5 in \citet{anagnostides2026computational}. In their work, they show that if the learner has exact access to the retraining operator $$\theta \mapsto \argmin_{\theta' \in \Theta} \DPR{\theta}{\theta'}$$
then one can compute a single model that is approximately performatively stable in $\poly(d,\log(1/\eps))$ many retraining steps if $\rho=L\beta/\gamma$ is just above 1. They operate in an idealized setting where the learner essentially knows the distribution map $\cD(\cdot)$ exactly, whereas our procedure works in finite samples and achieves the sharp dependence on the problem dimension. Also, our algorithm works, and returns a stable mixture, for all values of $\rho$ ($\rho$, in fact, may not even be defined) whereas theirs requires that $\rho$ is close to 1.

\subsection{Analysis}
\label{sec:algo-analysis}

The analysis can be viewed as a finite-sample extension of the EVI result of \citet{EVI}, which in turn is an instance of a more general efficient computational version of Sion's minimax theorem \citep{sigecom} borrowing ideas from the Ellipsoid-Against-Hope framework and its generalizations \citep{papadimitriou2008computing,farina2024polyhedral}. In particular, we apply the computational version of Sion's minimax theorem to turn the proof of \cref{thm:approximate-existence} into an algorithm.

\paragraph{The Key Ideas.} Recalling the key step in the proof of \cref{thm:approximate-existence},
\[
    \sup_{\theta\in\Theta}\inf_{\mu\in\Delta(\Theta)} \Phi(\mu, \theta) \le 0
\]
because for any $\theta$ one can construct efficiently a distribution $\mu_\theta$ such that $\Phi(\mu_\theta, \theta) \le 0$. The exact same construction can be applied to the linearization of $\Phi$, 
\[
    \Psi(\mu, \theta) \coloneqq \E_{\theta'\sim \mu} \E_{z\sim \cD(\theta')} [\nabla \ell(z; \theta')]^\top (\theta' - \theta),
\]
meaning that $\sup_{\theta\in\Theta}\inf_{\mu\in\Delta(\Theta)} \Psi(\mu, \theta) \le 0.$ 
In the idealized population-oracle setting, choosing
$\mu_\theta=\delta_\theta$ witnesses violation of a constraint in the
infeasible strict system
\[
    \Omega \coloneqq \{ \theta \in \Theta : \Psi(\mu, \theta) > 0
    \quad\forall \mu\in\Delta(\Theta) \}.
\]
Indeed, at a query $\theta_t$, write $g_t\coloneqq\E_{z\sim\cD(\theta_t)}
    [\nabla\ell(z;\theta_t)].$ Then
\[
    \Psi(\delta_{\theta_t},\theta)
    =\ip{g_t}{\theta_t-\theta},
    \qquad
    \Psi(\delta_{\theta_t},\theta_t)=0.
\]
If $g_t\neq0$, the closed halfspace
$\{\theta:\ip{g_t}{\theta-\theta_t}\leq0\}$ contains $\Omega$ and has
the query point on its boundary. This population picture motivates the
orientation of the deployment cut. If $g_t=0$, the point mass at $\theta_t$
is already an exact population EVI certificate.

This picture is only motivational. An empty infinite system of strict linear
inequalities need not contain a finite infeasible subsystem, so running the cutting-plane procedure for finitely many iterations does not imply that a finite analogue of $\Omega$ is literally empty.
Instead, the cutting-plane algorithm iteratively drives the volume of the intersection of all the cuts to be progressively smaller. The following lemma shows that when the volume of this set is smaller than the value of an $\ell_2$ ball of a certain radius, then the approximate spanning certificate (as defined formally in \Cref{eq:spanning}) must hold with respect to the series of cuts.

There is a second important distinction in the finite-sample algorithm: its
empirical halfspaces need not individually be valid separation replies for
$\Omega$. The formal Phase-I argument treats their normals as arbitrary
bounded vectors and proves pathwise that the finite transcript is
$\eps$-spanning. Phase II applies minimax duality to this transcript, and
concentration is used only afterward to transfer the empirical certificate
to the population operator.

\paragraph{Phase I: Constructing a Sparse Empirical Certificate.}

We now define the quantitative finite certificate used in the proof. Let
$K\geq1$ and let
\[
    \cV=\{(\theta_i,a_i)\}_{i=1}^K
    \subseteq\Theta\times\R^d
\]
be a putative transcript of deployed models and corresponding cuts produced
by the cutting-plane algorithm.
We
call $\cV$ \emph{$\eps$-spanning} if
\begin{equation}
    \max_{\theta'\in\Theta}\min_{i\in[K]}
    \ip{a_i}{\theta_i-\theta'}
    \leq\eps.
    \label{eq:spanning}
\end{equation}
That is, no model $\theta'\in\Theta$ satisfies every recorded
constraint with margin greater than $\eps$. In our algorithm,
$a_i=\widetilde\nabla_i$.  The residual region associated with the deployment cuts in $\cV$ is the intersection of their halfspaces:
\begin{equation}
    P(\cV)
    \coloneqq
    \Theta\cap
    \bigcap_{i=1}^K
    \{\theta':\ip{a_i}{\theta_i-\theta'}\geq0\}.
    \label{eq:residual-region}
\end{equation}
The following lemma gives the quantitative connection between the size of
this region and spanning.

\begin{lemma}[A nonspanning certificate leaves a ball]
\label{lem:nonspanning-ball}
Assume that $\Theta$ is a well-bounded convex body \eqref{eq:rounding} and $\norm{a_i}\leq G$ for every $i$. If $\cV$
is not $\eps$-spanning, then there is a Euclidean ball
$B\subseteq P(\cV)$ of radius at least
\[
    r_\eps
    \coloneqq
    \frac{r\eps}{3(\eps+GR)}.
\]
Moreover, the ball is strictly in the interior of the cuts:
\[
    \ip{a_i}{\theta_i-\theta'}>\frac{\eps}{3}
    \qquad
    \text{for every }\theta'\in B\text{ and }i\in[K].
\]
\end{lemma}

See \cref{app:proof-nonspanning-ball} for a proof.
The lemma proves that if the cutting-plane method shrinks the
maintained region below the volume of any radius-$r_\eps$ ball, the transcript of resulting cuts
must be $\eps$-spanning. The optimal-query volumetric method of
\citet{jiang2020improved} reaches this point in
$\cO(d\log(dGR^2/(r\eps)))$ rounds.

\paragraph{Concentration Analysis.}

To relate this empirical spanning certificate to the population EVI, fix any
$\lambda\in\Delta_K$ and denote its empirical gap by $\tau$, so that
\begin{equation}
    \widehat\Psi_{\cV}(\lambda)
    =
    \max_{\theta'\in\Theta}
    \sum_{i=1}^K\lambda_i
    \ip{\widetilde\nabla_i}{\theta_i-\theta'}
    \leq\tau.
    \label{eq:estimated-evi-certificate}
\end{equation}

\begin{lemma}[Estimated-to-population transfer]
\label{lem:estimated-evi-transfer}
Let $\lambda\in\Delta_K$ and
$\mu=\sum_{i=1}^K\lambda_i\delta_{\theta_i}$. Suppose that
\eqref{eq:estimated-evi-certificate} holds and
\[
    \norm{\widetilde\nabla_i-
    \E_{z\sim\cD(\theta_i)}[\nabla\ell(z;\theta_i)]}_2\leq\xi
    \qquad\text{for every }i.
\]
Then
\begin{equation}
    \operatorname{Gap}_{\rm EVI}(\mu)
    \leq\tau+2R\xi.
    \label{eq:estimated-to-population-evi}
\end{equation}
Consequently, if $\ell(z;\theta)$ is convex in $\theta$ for every $z\in\cZ$, then
$\operatorname{Gap}_{\rm PS}(\mu)\leq\tau+2R\xi$.  If $\ell(z;\cdot)$ is
$\gamma$-strongly convex on $\Theta$ for every $z\in\cZ$, then following the sharper bound holds:
\begin{equation}
    \operatorname{Gap}_{\rm PS}(\mu)
    \leq\tau+\frac{\xi^2}{2\gamma}.
    \label{eq:estimated-to-strongly-convex-stability}
\end{equation}
\end{lemma}

\begin{proof}
Fix a comparator $\theta'\in\Theta$. Adding and subtracting each empirical
gradient and applying Cauchy--Schwarz gives
\begin{align*}
    \sum_i\lambda_i
    \ip{\E_{z\sim\cD(\theta_i)}[\nabla\ell(z;\theta_i)]}
       {\theta_i-\theta'}
    &\leq
    \sum_i\lambda_i\ip{\widetilde\nabla_i}{\theta_i-\theta'} \\
    &+\sum_i\lambda_i
      \norm{\E_{z\sim\cD(\theta_i)}[\nabla\ell(z;\theta_i)]
      -\widetilde\nabla_i}_2 \cdot \norm{\theta_i-\theta'}_2\\
    &\leq\tau+2R\xi.
\end{align*}
Here we used $\norm{\theta_i-\theta'}\leq2R$, which follows from
$\Theta\subseteq B_2(0,R)$. 
Since the inequality holds for every $\theta'$, maximizing over
$\theta'\in\Theta$ proves \eqref{eq:estimated-to-population-evi}, a bound on the EVI-gap of the mixture. The bound on the stability gap follows from \Cref{prop:evi-implies-stability}.

Now suppose that the loss is pointwise $\gamma$-strongly convex. For every $i$ and
$\theta'$,
\begin{align*}
    \tDPR(\theta_i,\theta_i)-\tDPR(\theta_i,\theta')
    &\leq
    \ip{\E_{z\sim\cD(\theta_i)}[\nabla\ell(z;\theta_i)]}
       {\theta_i-\theta'}
    -\frac\gamma2\norm{\theta_i-\theta'}_2^2\\
    &\leq
    \ip{\widetilde\nabla_i}{\theta_i-\theta'}
    +\frac{\norm{\E_{z\sim\cD(\theta_i)}[\nabla\ell(z;\theta_i)]
    -\widetilde\nabla_i}_2^2}{2\gamma}.
\end{align*}
The second inequality uses
$\ip{u}{v}-(\gamma/2)\norm{v}_2^2\leq\norm{u}_2^2/(2\gamma)$, with $u$ equal to the error
(population gradient minus the empirical gradient) and
$v=\theta_i-\theta'$. Averaging over $i$ with weights $\lambda_i$ and then
maximizing over $\theta'$ proves
\eqref{eq:estimated-to-strongly-convex-stability}.
\end{proof}

It remains to control the gradient errors uniformly despite the adaptive
deployment points which we do via the following lemma. The proof follows directly from Hoeffding's inequality and a union bound, and is given in \cref{app:adaptive-concentration}.

\begin{lemma}[Uniform concentration over adaptive deployments]
\label{lem:adaptive-gradient-concentration}
Fix a deterministic cutting-plane cap $T$, an accuracy $\xi>0$, a confidence
level $\delta\in(0,1)$, and a batch size $n\geq1$.  If
\begin{equation}
    n\geq 8
    \frac{G^2}{\xi^2}
    \log\frac{2T}{\delta}
    \label{eq:adaptive-sample-size}
\end{equation}
then, with probability at least $1-\delta$,
\[
    \norm{\widetilde\nabla_i-
    \E_{z\sim\cD(\theta_i)}[\nabla\ell(z;\theta_i)]}_2\leq\xi
\]
at every deployment occurring in the first $T$ cutting-plane rounds.
\end{lemma}

\paragraph{Phase II: Extracting the Mixture.}

Phase I leaves us with a finite constraint system. Phase II converts it into
a dual certificate using minimax duality and only the recorded query points
and gradients. For $\lambda\in\Delta_K$ and $\theta'\in\Theta$, recall the empirical value of the game
\[
    \widehat{\Psi}_{\cV}(\lambda,\theta')
    \coloneqq
    \sum_{i=1}^K\lambda_i\ip{a_i}{\theta_i-\theta'}.
\]
We now show that if $\cV$ is $\eps$-spanning, then one can extract weights over the elements of $\cV$ that produce the desired mixture.
\begin{lemma}[Spanning--mixture duality]
\label{lem:spanning-mixture-duality}
If $\cV$ is $\eps$-spanning, then there exists
$\lambda\in\Delta_K$ such that
\begin{equation}
    \max_{\theta'\in\Theta}
   \widehat{\Psi}_{\cV}(\lambda,\theta')
    \leq\eps.
    \label{eq:finite-mixture-certificate}
\end{equation}
If the minimization over $\lambda$ is solved to additive accuracy
$\eps$, the right-hand side becomes $2\eps$.
\end{lemma}

\begin{proof}
The sets $\Delta_K$ and $\Theta$ are compact and convex, and $\widehat{\Psi}_{\cV}$ is
affine in both arguments. Therefore, Sion's minimax theorem
\citep[Theorem~3.4]{sion1958general} gives
\begin{align*}
    \min_{\lambda\in\Delta_K}\max_{\theta'\in\Theta}
     \widehat{\Psi}_{\cV}(\lambda,\theta')
    &=\max_{\theta'\in\Theta}\min_{\lambda\in\Delta_K}
     \widehat{\Psi}_{\cV}(\lambda,\theta')\\
    &=\max_{\theta'\in\Theta}\min_{i\in[K]}
      \ip{a_i}{\theta_i-\theta'}
    \leq\eps.
\end{align*}
The second equality follows because a linear function over the simplex
attains its minimum at a vertex. Thus, an exact minimizer satisfies
\eqref{eq:finite-mixture-certificate}; an additive-$\eps$ minimizer increases
the right-hand side to $2\eps$.
\end{proof}

This is the finite-support extraction step of the EVI construction of
\citet{EVI}. 
The weights $\lambda$ can be computed
offline in  time $\poly(K,G,R,1/\eps)$  while performing no further deployments. See \citet{EVI}.

\paragraph{Putting the Pieces Together.}

We now combine the deterministic empirical certificate with the uniform
concentration event.

\begin{proof}[Proof of \cref{thm:finite-sample-master}]
Run \cref{alg:few-deployment} with the round cap obtained from the
residual-ball and volumetric argument in
\cref{app:volumetric-localization}, at accuracy $\eps/3$; the factor of $3$
is absorbed by the implicit constant. That argument gives an $(\eps/3)$-spanning family
$\{(\theta_i,\widetilde\nabla_i)\}_{i=1}^K$ with $1\leq K\leq T$. We then
compute the weights to additive accuracy
$\eps/3$. By \cref{lem:spanning-mixture-duality}, the resulting
$\lambda\in\Delta_K$ defines a mixture
$\mu=\sum_{i=1}^K\lambda_i\delta_{\theta_i}$ satisfying
\begin{equation}
    \widehat\Psi_{\cV}(\lambda)
    =
    \max_{\theta'\in\Theta}
    \sum_{i=1}^K\lambda_i
    \ip{\widetilde\nabla_i}{\theta_i-\theta'}
    \leq\frac{2\eps}{3}.
    \label{eq:empirical-first-order-certificate}
\end{equation}

For case~\textup{(i)}, set $\xi=\eps/(6R)$. For case~\textup{(ii)}, set
$\xi^2=2\gamma\eps/3$. In either case, the corresponding assumption on $n$,
together with \cref{lem:adaptive-gradient-concentration} applied to the
deterministic cap $T$, shows that with probability at least $1-\delta$,
\begin{equation}
    \norm{\widetilde\nabla_i-
    \E_{z\sim\cD(\theta_i)}[\nabla\ell(z;\theta_i)]}_2\leq\xi
    \qquad\text{for every deployed model }\theta_i.
    \label{eq:uniform-master-gradient-event}
\end{equation}

In case~\textup{(i)}, \cref{lem:estimated-evi-transfer} with
$\tau=2\eps/3$ gives
\[
    \operatorname{Gap}_{\rm EVI}(\mu)
    \leq\frac{2\eps}{3}+2R\xi
    =\eps.
\]
Thus, $\mu$ is $\eps$-performatively EVI-stable and, by
\cref{prop:evi-implies-stability}, also $\eps$-performatively stable. 

In case~\textup{(ii)}, the strongly convex conclusion of
\cref{lem:estimated-evi-transfer} gives
\[
    \operatorname{Gap}_{\rm PS}(\mu)
    \leq\frac{2\eps}{3}+\frac{\xi^2}{2\gamma}
    =\eps.
\]
The general part of the same lemma also shows explicitly that
\[
    \operatorname{Gap}_{\rm EVI}(\mu)
    \leq\frac{2\eps}{3}+2R\xi
    =\frac{2\eps}{3}+2R\sqrt{\frac{2\gamma\eps}{3}}.
\]
Finally, if the strongly convex loss is run with the batch size in
case~\textup{(i)}, set $\xi=\eps/(6R)$ instead. The case~\textup{(i)}
argument then proves $\operatorname{Gap}_{\rm EVI}(\mu)\leq\eps$; strong
convexity is not needed for this last implication.
\end{proof}

The three ingredients that make up the construction---the cutting-plane algorithm, the estimation of gradients via samples, and the offline computation of the atomic mixture---all serve separate roles. The first controls the number of deployments, the second controls the batch
size at each deployment, and the third is entirely offline and done to extract a mixture from the cuts produced.

\section{A Deployment Lower Bound}
\label{sec:lower-bound}

Finally, we show that the deployment guarantee from the previous section is
essentially tight. Rather than constructing a hard performative prediction instance from scratch, the main observation behind our lower bound is that performative prediction contains
black-box convex optimization as a special case. In particular, we can choose
the distribution map so that deploying a model reveals exactly the information
returned by a first-order oracle. Lower bounds on the number of oracle queries
required for convex optimization then translate directly into lower bounds on
the number of model deployments.

\begin{theorem}[Deployment lower bound]
\label{thm:deployment-lower-bound}
There exist universal constants $c_1,c_2>0$ such that the following holds for
every $d\geq2$ and $\eps \in (0, c_2 / d]$.
Any possibly randomized algorithm that, on every $d$-dimensional instance,
returns an $\eps$-performatively stable mixture with probability at least $2/3$
must in the worst case deploy at least
\[
    c_1\min\left\{
        \frac{1}{\eps^2},
        d\log\frac{1}{\eps}
    \right\}.
\]
many models $\theta$. The lower bound holds even when $\Theta=B_2(0,1)$, the loss $\ell(z;\theta)$
is affine in $\theta$, uniformly bounded by a constant, and every distribution
$\cD(\theta)$ is supported on a single point.
\end{theorem}

\begin{proof}
We begin by showing how black-box convex optimization can be embedded into  the performative
prediction framework. Let $\Theta=B_2(0,1)$  and consider the problem of minimizing a convex $1$-Lipschitz function $f:\Theta\to\R$ given access to a first-order oracle. In the oracle model of
\citet[Section~II-B]{braun2017lower}, querying a point $\theta\in\Theta$
returns $\bigl(f(\theta),g_f(\theta)\bigr)$ where $ g_f(\theta)\in\partial f(\theta)$ is a subgradient of $f$ at $\theta$.


Fix a function $f$ together with this first-order oracle. We associate with it
the following performative prediction instance:
\[
    \cZ=\R\times B_2(0,1),
    \qquad
    \cD_f(\theta)
    =
    \delta_{(f(\theta),g_f(\theta))},
    \qquad
    \ell((a,v);\theta')
    =
    1+\ip{v}{\theta'}.
\]
The first coordinate $a$ does not enter the loss. Its only purpose is to make
the observation produced by deploying $\theta$ equal to the full first-order
oracle response $(f(\theta),g_f(\theta))$. Thus, one model deployment can be
simulated using exactly one oracle query, and conversely the learner receives
no more information from a deployment than the corresponding first-order
oracle response. The loss $\ell$ is affine in its model argument and since
$\theta',g_f(\theta)\in B_2(0,1)$, the loss is bounded by 2.

We next show that finding a performatively stable mixture on this instance yields a solution to the original convex optimization problem. Let $\mu$ be any distribution over $\Theta$, and let $\theta_f^\star\in\argmin_{\theta\in\Theta}f(\theta).$ For the performative instance above $\tDPR(\theta,\theta') = 1+\ip{g_f(\theta)}{\theta'}$ and
\begin{align*}
    \operatorname{Gap}_{\rm PS}(\mu)
    &=
    \max_{\theta'\in\Theta}
    \E_{\theta\sim\mu}
    \left[
        \tDPR(\theta,\theta)-\tDPR(\theta,\theta')
    \right]\\
    &=
    \max_{\theta'\in\Theta}
    \E_{\theta\sim\mu}
    \left[
        \ip{g_f(\theta)}{\theta-\theta'}
    \right]\\
    &\geq
    \E_{\theta\sim\mu}
    \left[
        \ip{g_f(\theta)}{\theta-\theta_f^\star}
    \right]\\
    &\geq
    \E_{\theta\sim\mu}
    \left[
        f(\theta)-f(\theta_f^\star)
    \right],
\end{align*}
where the final inequality follows from convexity of $f$.
Consequently, every $\eps$-performatively stable mixture satisfies
\[
    \E_{\theta\sim\mu}[f(\theta)]-f(\theta_f^\star)
    \leq\eps.
\]
Moreover, because $\Theta$ is convex, $\bar\theta
    \coloneqq
    \E_{\theta\sim\mu}[\theta]
    \in\Theta$, 
and Jensen's inequality gives
\[
    f(\bar\theta)-f(\theta_f^\star)
    \leq
    \E_{\theta\sim\mu}[f(\theta)]-f(\theta_f^\star)
    \leq\eps.
\]
Thus, from any $\eps$-performatively stable mixture we can recover an
$\eps$-optimal solution to the original convex optimization problem simply by
taking its mean.
It follows that any possibly randomized performative prediction algorithm
using $T$ deployments implies a randomized first-order optimization algorithm
using exactly $T$ oracle queries.

We can now apply the oracle lower bound of
\citet[Proposition~V.5]{braun2017lower}. For every $1\leq p<\infty$,
$\delta>0$, and accuracy $\eta\leq d^{-1/p-\delta}$, they construct a family of convex $1$-Lipschitz functions on
$B_p(0,1)$ such that,
for algorithms with error probability at most $P_e$, the distributional
oracle complexity of finding an $\eta$-minimum is $\Omega\left(
        (1-P_e)d\log\frac{1}{\eta}
    \right).$
Specializing to $p=2$, $\delta=1/2$, and $P_e=1/3$ gives $\Omega\left(d\log\frac{1}{\eta}\right)$
oracle queries whenever $\eta\leq d^{-1}$. After adjusting universal constants, this implies
\[
    c_1\min\left\{
        \frac{1}{\eps^2},
        d\log\frac{1}{\eps}
    \right\}
\]
deployments for some universal $c_1>0$, completing the proof.
\end{proof}


\printbibliography

@article{braun2017lower,
  title   = {Lower Bounds on the Oracle Complexity of Nonsmooth Convex
             Optimization via Information Theory},
  author  = {Braun, G{\'a}bor and Guzm{\'a}n, Crist{\'o}bal and Pokutta, Sebastian},
  journal = {IEEE Transactions on Information Theory},
  year    = {2017},
}

@book{groetschel1988geometric,
  title     = {Geometric Algorithms and Combinatorial Optimization},
  author    = {Gr{\"o}tschel, Martin and Lov{\'a}sz, L{\'a}szl{\'o} and Schrijver, Alexander},
  series    = {Algorithms and Combinatorics},
  year      = {1988},
  publisher = {Springer},
}

@inproceedings{perdomo2020performative,
  title     = {Performative Prediction},
  author    = {Perdomo, Juan and Zrnic, Tijana and Mendler-D{\"u}nner, Celestine and Hardt, Moritz},
  booktitle = {International Conference on Machine Learning},
  year      = {2020}
}

@article{perdomo2025revisiting,
  title   = {Revisiting the Predictability of Performative, Social Events},
  author  = {Perdomo, Juan C.},
  journal = {International Conference on Machine Learning},
  year    = {2025}
}

@article{freund1999adaptive,
  title   = {Adaptive Game Playing Using Multiplicative Weights},
  author  = {Freund, Yoav and Schapire, Robert E.},
  journal = {Games and Economic Behavior},
  year    = {1999},
}

@inproceedings{EVI,
  title     = {Expected Variational Inequalities},
  author    = {Zhang, Brian Hu and Anagnostides, Ioannis and Tewolde, Emanuel and Berker, Ratip Emin and Farina, Gabriele and Conitzer, Vincent and Sandholm, Tuomas},
  booktitle = {International Conference on Machine Learning},
  year      = {2025}
}

@article{lee2026partially,
  title={Partially Performative Prediction},
  author={Lee, Jaewook and Zrnic, Tijana},
  journal={arXiv preprint arXiv:2606.07890},
  year={2026}
}

@inproceedings{hardt2026future,
  title={The Future of Prediction: In Conversation with Oskar Morgenstern},
  author={Hardt, Moritz},
  booktitle={International Congress of Mathematicians 2026},
  year={2026},
  organization={SIAM}
}

@inproceedings{miller2021outside,
  title={Outside the echo chamber: Optimizing the performative risk},
  author={Miller, John P and Perdomo, Juan C and Zrnic, Tijana},
  booktitle={International Conference on Machine Learning},
  year={2021},
}

@article{kim2022making,
  title={Making decisions under outcome performativity},
  author={Kim, Michael P and Perdomo, Juan C},
  journal={Innovations in Theoretical Computer Science},
  year={2023}
}

@inproceedings{izzo2021learn,
  title={How to learn when data reacts to your model: performative gradient descent},
  author={Izzo, Zachary and Ying, Lexing and Zou, James},
  booktitle={International Conference on Machine Learning},
  year={2021},
}

@inproceedings{jagadeesan2022regret,
  title={Regret minimization with performative feedback},
  author={Jagadeesan, Meena and Zrnic, Tijana and Mendler-D{\"u}nner, Celestine},
  booktitle={International Conference on Machine Learning},
  year={2022},
}

@InProceedings{pluginpp,
  title = 	 {Plug-in Performative Optimization},
  author =       {Lin, Licong and Zrnic, Tijana},
  booktitle = 	 {International Conference on Machine Learning},
  year = 	 {2024},
}

@article{grunberg1954predictability,
  title   = {The predictability of social events},
  author  = {Grunberg, Emile and Modigliani, Franco},
  journal = {Journal of Political Economy},
  year    = {1954}
}

@article{li2022multi,
  title={Multi-agent performative prediction with greedy deployment and consensus seeking agents},
  author={Li, Qiang and Yau, Chung-Yiu and Wai, Hoi-To},
  journal={Advances in Neural Information Processing Systems},
  year={2022}
}

@inproceedings{piliouras2023multi,
  title={Multi-agent performative prediction: From global stability and optimality to chaos},
  author={Piliouras, Georgios and Yu, Fang-Yi},
  booktitle={ACM Conference on Economics and Computation},
  year={2023}
}

@article{papadimitriou2008computing,
  title={Computing correlated equilibria in multi-player games},
  author={Papadimitriou, Christos H and Roughgarden, Tim},
  journal={Journal of the ACM},
  year={2008},
}

@article{narang2023multiplayer,
  title={Multiplayer performative prediction: Learning in decision-dependent games},
  author={Narang, Adhyyan and Faulkner, Evan and Drusvyatskiy, Dmitriy and Fazel, Maryam and Ratliff, Lillian J},
  journal={Journal of Machine Learning Research},
  year={2023}
}

@article{sigecom,
  title   = {Turning defense into offense in {$O(\log 1/\epsilon)$} steps: Efficient constructive proof of the minimax theorem},
  author  = {Gabriele Farina},
  year    = {2026},
  journal = {ACM SIGecom Exchanges}
}

@inproceedings{hardt2016strategic,
  title     = {Strategic classification},
  author    = {Hardt, Moritz and Megiddo, Nimrod and Papadimitriou, Christos and Wootters, Mary},
  booktitle = {Innovations in Theoretical Computer Science},
  year      = {2016}
}

@article{anagnostides2026computational,
  title   = {On the Computational Complexity of Performative Prediction},
  author  = {Anagnostides, Ioannis and Chauhan, Rohan and Panageas, Ioannis and Sandholm, Tuomas and Yan, Jingming},
  journal = {International Conference on Machine Learning},
  year    = {2026}
}

@inproceedings{dews,
  author    = {Perdomo, Juan Carlos and Britton, Tolani and Hardt, Moritz and Abebe, Rediet},
  title     = {Difficult Lessons on Social Prediction from Wisconsin Public Schools},
  year      = {2025},
  booktitle = {ACM Conference on Fairness, Accountability, and Transparency}
}

@article{mendler2024engine,
  title   = {An engine not a camera: Measuring performative power of online search},
  author  = {Mendler-D{\"u}nner, Celestine and Carovano, Gabriele and Hardt, Moritz},
  journal = {Advances in Neural Information Processing Systems},
  year    = {2024}
}

@article{farina2026stability,
  title   = {The stability of online algorithms in performative prediction},
  author  = {Farina, Gabriele and Perdomo, Juan Carlos},
  journal = {arXiv preprint arXiv:2602.24207},
  year    = {2026}
}

@inproceedings{brown2022performative,
  title={Performative prediction in a stateful world},
  author={Brown, Gavin and Hod, Shlomi and Kalemaj, Iden},
  booktitle={International Conference on Artificial Intelligence and Statistics},
  year={2022},
}

@book{morgenstern1928wirtschaftsprognose,
  title     = {Wirtschaftsprognose: Eine Untersuchung ihrer Voraussetzungen und M{\"o}glichkeiten},
  author    = {Morgenstern, Oskar},
  year      = {1928},
  publisher = {Springer}
}

@article{simon1954bandwagon,
  title   = {Bandwagon and underdog effects and the possibility of election predictions},
  author  = {Simon, Herbert A},
  journal = {Public Opinion Quarterly},
  year    = {1954}
}

@article{mendler2020stochastic,
  title   = {Stochastic optimization for performative prediction},
  author  = {Mendler-D{\"u}nner, Celestine and Perdomo, Juan and Zrnic, Tijana and Hardt, Moritz},
  journal = {Advances in Neural Information Processing Systems},
  year    = {2020}
}

@article{drusvyatskiy2023stochastic,
  title   = {Stochastic optimization with decision-dependent distributions},
  author  = {Drusvyatskiy, Dmitriy and Xiao, Lin},
  journal = {Mathematics of Operations Research},
  year    = {2023}
}

@article{chyn2021returns,
  title={The returns to early-life interventions for very low birth weight children},
  author={Chyn, Eric and Gold, Samantha and Hastings, Justine},
  journal={Journal of Health Economics},
  year={2021},
}

@article{merton1948self,
  title={The self-fulfilling prophecy},
  author={Merton, Robert K},
  journal={The Antioch Review},
  year={1948},
}

@article{hardt2026retraining,
  title={Retraining Seeks Stable Signals},
  author={Hardt, Moritz},
  journal={arXiv preprint arXiv:2607.15623},
  year={2026}
}

@article{sion1958general,
  author  = {Sion, Maurice},
  title   = {On General Minimax Theorems},
  journal = {Pacific Journal of Mathematics},
  year    = {1958},
}

@article{foster2021forecast,
  title   = {Forecast Hedging and Calibration},
  author  = {Foster, Dean P. and Hart, Sergiu},
  journal = {Journal of Political Economy},
  year    = {2021},
}

@article{nemirovski2010accuracy,
  title   = {Accuracy Certificates for Computational Problems with Convex Structure},
  author  = {Nemirovski, Arkadi and Onn, Shmuel and Rothblum, Uriel G.},
  journal = {Mathematics of Operations Research},
  year    = {2010},
}

@article{perdomo2025defense,
  title={In defense of defensive forecasting},
  author={Perdomo, Juan Carlos and Recht, Benjamin},
  journal={arXiv preprint arXiv:2506.11848},
  year={2025}
}

@inproceedings{daskalakis2025efficient,
  title={Efficient learning and computation of linear correlated equilibrium in general convex games},
  author={Daskalakis, Constantinos and Farina, Gabriele and Fishelson, Maxwell and Pipis, Charilaos and Schneider, Jon},
  booktitle={ACM Symposium on Theory of Computing},
  year={2025}
}

@inproceedings{jiang2020improved,
  author    = {Jiang, Haotian and Lee, Yin Tat and Song, Zhao and Wong, Sam Chiu-wai},
  title     = {An Improved Cutting Plane Method for Convex Optimization, Convex--Concave Games, and Its Applications},
  booktitle = {ACM SIGACT Symposium on Theory of Computing},
  year      = {2020},
}

@inproceedings{farina2024polyhedral,
  title     = {Polynomial-Time Computation of Exact {$\Phi$}-Equilibria in Polyhedral Games},
  author    = {Farina, Gabriele and Pipis, Charilaos},
  booktitle = {Advances in Neural Information Processing Systems},
  year      = {2024}
}

\appendix

\section{Proofs for the Structural Results}
\label{app:structural-proofs}

Here we flesh out the proofs deferred from \cref{sec:structural}. The first two
constructions show why convexity is necessary and why exact
stability need not be attained. We then prove the relationship between
performative stability and EVI-stability for point masses and give the two
counterexamples for approximate solutions.

\begin{proof}[Proof of \cref{prop:no-exact-mixture}]
Let $\Theta=[-1,1]$, let $\cZ=\{-1,1\}$, and set
\[
    s(\theta)=
    \begin{cases}
        -1, & \theta<0,\\
        +1, & \theta\geq0,
    \end{cases}
    \qquad
    \cD(\theta)=\delta_{s(\theta)},
    \qquad
    \ell(z;\theta')=1+z\theta'.
\]
The loss is bounded, jointly continuous, and affine in $\theta'$. The
decoupled performative risk is
$\tDPR(\theta,\theta')=1+s(\theta)\theta'$. Thus, if $\theta\sim\mu$,
\begin{align*}
    \operatorname{Gap}_{\rm PS}(\mu)
    &=\max_{\theta'\in[-1,1]}
      \bigl\{\E[s(\theta)\theta]-\theta'\E[s(\theta)]\bigr\}\\
    &=\E|\theta|+|\E s(\theta)|.
\end{align*}
This gap is positive for every $\mu$. Indeed, the claim is immediate if
$\E|\theta|>0$, while $\E|\theta|=0$ implies that $\theta=0$ almost surely and hence
$\E s(\theta)=1$. On the other hand, for any $a\in(0,1]$, the uniform distribution
over $\{-a,a\}$ has gap exactly $a$. Therefore, the infimum of the gap is zero,
but no distribution attains it.
\end{proof}

\begin{proof}[Proof of \cref{prop:nonconvex-no-mixture}]
Let $\Theta=\cZ=[0,1]$, $\cD(\theta)=\delta_\theta$, and
$\ell(z;\theta')=1-(\theta'-z)^2$. The loss takes values in $[0,1]$ and is
concave in $\theta'$. Moreover, $\tDPR(\theta,\theta)=1$ for every
$\theta\in\Theta$. Hence, any distribution $\mu$ satisfies
\begin{align*}
    \operatorname{Gap}_{\rm PS}(\mu)
    &=\max_{\theta'\in[0,1]}\E_{\theta\sim\mu}(\theta'-\theta)^2\geq
      \frac12\E[\theta^2+(1-\theta)^2]
    \geq\frac14,
\end{align*}
where the first inequality lower-bounds the maximum by the average of the
choices $\theta'=0$ and $\theta'=1$, and the second uses
$\theta^2+(1-\theta)^2\geq1/2$. Thus, no mixture can have stability gap less
than $1/4$.
\end{proof}

\begin{proof}[Proof of \Cref{prop:stability-not-evi}]
For a fixed $\theta\in\Theta$,
\[
    \operatorname{Gap}_{\rm PS}(\delta_\theta)
    = \tDPR(\theta, \theta)-\min_{\theta' \in \Theta} \tDPR(\theta, \theta').
\]
This quantity is zero if and only if $\theta$ minimizes the differentiable
convex function $F(\theta')=\tDPR(\theta,\theta')$. By the first-order
optimality condition for convex functions, this is equivalent to
\[
    \ip{\E_{z \sim \cD(\theta)} \nabla \ell(z;\theta)}{\theta'-\theta}\geq0
\]
for every $\theta'\in\Theta$. This is precisely the condition
$\EVIGap(\delta_\theta)=0$, proving the first statement.

For part~\textup{(i)}, consider the following one-dimensional instance:
\[
    \Theta=\cZ=[0,1],\qquad
    \ell(z;\theta')=\frac12(\theta'-z)^2,
    \qquad
    \cD(\theta)=\delta_{2\theta/3},
    \qquad
    \mu=\frac14\delta_0+\frac34\delta_1.
\]
On the unit interval, the loss is bounded by $1/2$, differentiable, convex,
and $1$-Lipschitz in its model argument. The induced risk and population
gradient are
\[
    \tDPR(\theta,\theta')=\frac12\left(\theta'-\frac23\theta\right)^2,
    \qquad
    \E_{z\sim\cD(\theta)}[\nabla\ell(z;\theta)]
    =\frac13\theta.
\]
We can now compute the performative stability gap directly:
\begin{align*}
    \min_{\theta' \in \Theta}\E_{\theta\sim\mu}\tDPR(\theta,\theta')
    = \min_{\theta' \in \Theta} \frac18(\theta')^2+\frac38\left(\theta'-\frac23\right)^2 = \frac1{24}\\
    \E_{\theta\sim\mu}\tDPR(\theta,\theta)
    =\frac14\tDPR(0,0)+\frac34\tDPR(1,1)
    =\frac1{24}.
\end{align*}
Thus, $\operatorname{Gap}_{\rm PS}(\mu)=0$. On the other hand, the EVI
expression is
\[
    \E_{\theta\sim\mu}
    \ip{\E_{z\sim\cD(\theta)}[\nabla\ell(z;\theta)]}{\theta-\theta'}
    =\frac34\cdot\frac13(1-\theta')
    =\frac14(1-\theta').
\]
Taking $\theta'=0$ maximizes this expression and gives
$\operatorname{Gap}_{\rm EVI}(\mu)=1/4$.

For part~\textup{(ii)}, let $\Theta=\cZ=[0,1]$,
$\cD(\theta)=\delta_\theta$, and define
\[
\ell(z;\theta')
\coloneqq
\begin{cases}
0,
    &z\leq\frac12,\\[2pt]
0,
    &z>\frac12 \text{ and } \theta'\leq1-z,\\[2pt]
\dfrac{(\theta'+z-1)^2}{2(2z-1)},
    &z>\frac12 \text{ and } 1-z<\theta'<z,\\[7pt]
\theta'-\frac12,
    &z>\frac12 \text{ and } \theta'\geq z.
\end{cases}
\]
For each fixed $z>1/2$, the derivative with respect to $\theta'$ in the three
nontrivial regions is, respectively,
\[
    0,
    \qquad
    \frac{\theta'+z-1}{2z-1},\text{and}
    \qquad
    1.
\]
This derivative is continuous, nondecreasing, and lies in $[0,1]$; hence,
$\ell(z;\cdot)$ is differentiable, convex, and $1$-Lipschitz. The loss is
also bounded in $[0,1/2]$.

Fix any $\alpha\in(0,1/2]$ and set
$\theta_\alpha=1/2+\alpha$. Deploying $\theta_\alpha$ produces
$z=\theta_\alpha$ deterministically, and the induced risk has minimum zero.
Therefore,
\[
    \operatorname{Gap}_{\rm PS}(\delta_{\theta_\alpha})
    =\ell(\theta_\alpha;\theta_\alpha)
    =\theta_\alpha-\frac12
    =\alpha.
\]
At the deployed point,
$\nabla \ell(\theta_\alpha;\theta_\alpha)=1$. It follows that
\[
    \operatorname{Gap}_{\rm EVI}(\delta_{\theta_\alpha})
    =\max_{\theta'\in[0,1]}(\theta_\alpha-\theta')
    =\theta_\alpha
    \geq\frac12
\]
as claimed.
\end{proof}

\section{Omitted Proofs in Algorithm Analysis}
\label{app:few-deployment-details}

In this appendix, we flesh out the remaining details in the analysis of \cref{alg:few-deployment}.

\subsection{Proof of \cref{lem:nonspanning-ball}}
\label{app:proof-nonspanning-ball}

\begin{proof}[Proof of \cref{lem:nonspanning-ball}]
If $\cV$ is not $\eps$-spanning, attainment of the maximum in
\eqref{eq:spanning} gives $\theta_0\in\Theta$ such that
\begin{equation}
    \ip{a_i}{\theta_i-\theta_0}>\eps
    \qquad\text{for every }i.
    \label{eq:strict-margin-u0}
\end{equation}
Set
\[
    \eta\coloneqq
    \frac{\eps}{3(\eps+GR)}
    \leq\frac13,
    \qquad
    \bar\theta=(1-\eta)\theta_0+\eta c.
\]
By convexity and the condition in \eqref{eq:rounding},
\begin{equation}
    B_2(\bar\theta,\eta r)\subseteq\Theta.
    \label{eq:contracted-inball}
\end{equation}
Indeed, every point $\theta'$ in this ball can be written as
$\theta'=(1-\eta)\theta_0+\eta\theta''$ for some
$\theta''\in B_2(c,r)\subseteq\Theta$, and hence belongs to
$\Theta$. Moreover,
\[
    \norm{\theta'-\theta_0}_2
    =\eta\norm{\theta''-\theta_0}_2
    \leq2\eta R,
\]
where the last inequality follows because $\theta'',\theta_0\in
\Theta\subseteq B_2(0,R)$.
Combining this inequality with \eqref{eq:strict-margin-u0} and applying
Cauchy--Schwarz gives, for every $i$,
\begin{align*}
    \ip{a_i}{\theta_i-\theta'}
    &>\eps-2G\eta R\\
    &\geq\eps/3>0.
\end{align*}
Thus, the closed ball $B_2(\bar\theta,\eta r)$ lies in $P(\cV)$. Moreover,
\[
    \eta r
    =\frac{r\eps}{3(\eps+GR)},
\]
as claimed.
\end{proof}

\subsection{Details on Cutting-Plane Algorithm}
\label{app:volumetric-localization}


\paragraph{Guarantees from \citet{jiang2020improved}.}
Since $\Theta\subseteq B_2(0,R)$, the box
$Q_0=[-R,R]^d$ contains $\Theta$. Fix a tolerance
$0<s\leq R$. We use the volumetric cutting-plane method analyzed in Appendix~A of
\citet{jiang2020improved}. Their
Theorem~4.1 shows that, for any convex target set contained in
$Q_0=[-R,R]^d$, the method uses
\[
    O\left(d\log\frac{dR}{s}\right)
\]
separation-oracle queries to either find a point in the target set or
certify that it contains no Euclidean ball of radius $s$. We compute the
leverage scores exactly, so the cutting-plane procedure itself is
deterministic; we suppress the polynomial arithmetic cost of computing the leverage scores and other overhead. In our application, \cref{lem:nonspanning-ball} allows us to set
\[
    s
    =
    \frac{r\eps}{3(\eps+GR)}.
\]
Therefore,
\[
    \frac{dR}{s}
    =
    \frac{3dR}{r}
    \left(1+\frac{GR}{\eps}\right),
\]
and the cutting-plane query bound becomes
\[
    T
    =
    O\left(
        d\log\left[
            \frac{3dR}{r}
            \left(1+\frac{GR}{\eps}\right)
        \right]
    \right).
\]

\section{Omitted Concentration Proof}
\label{app:adaptive-concentration}

We now give the details of the concentration argument behind
\cref{lem:adaptive-gradient-concentration}. Although the model deployed at a
given round may depend on the entire preceding history, conditional on that
history the deployed model is fixed and its batch consists of fresh,
independent samples. We can therefore apply a vector-valued concentration
inequality conditionally at each round and then union bound over the
deterministic round cap.

\begin{proof}[Proof of \cref{lem:adaptive-gradient-concentration}]
Index the cutting-plane rounds by $t\leq T$, and let $\mathcal G_t$ denote the
$\sigma$-algebra generated by the entire history before sampling the batch at
round $t$, including the current query. Let $A_t$ denote the event that round
$t$ triggers a deployment. Then $A_t$ is $\mathcal G_t$-measurable, and on
$A_t$ the deployed model $\theta_t$ is also $\mathcal G_t$-measurable. Conditional on $\mathcal G_t$ and on $A_t$, define
\[
    Y_{t,j}
    \coloneqq
    \nabla\ell(z_{t,j};\theta_t),
    \qquad
    g_t
    \coloneqq
    \E[Y_{t,j}\mid\mathcal G_t]
    =
    \E_{z\sim\cD(\theta_t)}
    [\nabla\ell(z;\theta_t)].
\]
Then the vectors
\[
    X_{t,j}
    \coloneqq
    Y_{t,j}-g_t
\]
are conditionally independent and mean zero. Moreover,
\[
    \norm{X_{t,j}}_2
    \leq
    \norm{Y_{t,j}}_2+\norm{g_t}_2
    \leq
    2G
\]
by \eqref{eq:sample-gradient-bound}. A standard Hoeffding inequality
therefore gives
\[
    \Pr\left(
        \norm{
            \frac1n
            \sum_{j=1}^n X_{t,j}
        }_2
        >\xi
        \,\middle|\,
        \mathcal G_t
    \right)
    \leq
    2\exp\left(
        -\frac{n\xi^2}{8G^2}
    \right)
    \qquad
    \text{on }A_t.
\]
Hence, if $n
    \geq
    \frac{8G^2}{\xi^2}
    \log\frac{2T}{\delta},
$
then
\[
    \Pr\left(
        \norm{
            \widetilde\nabla_t
            -
            \E_{z\sim\cD(\theta_t)}
            [\nabla\ell(z;\theta_t)]
        }_2
        >\xi
        \,\middle|\,
        \mathcal G_t
    \right)
    \leq
    \frac{\delta}{T}
\]
on $A_t$. Let $E_t$ be the event that round $t$ triggers a deployment and the above
gradient error exceeds $\xi$. Since $A_t\in\mathcal G_t$, the tower property
gives $\Pr(E_t)
    \leq
    \frac{\delta}{T}.$
A union bound over $t\leq T$ therefore shows that, with probability at least
$1-\delta$, the desired bound holds simultaneously at every deployment.
\end{proof}
\end{document}